\documentclass[twoside]{article}

\usepackage[accepted]{aistats2019}
\usepackage[round]{natbib}

\usepackage{amssymb,amsmath,amsthm,bbm}
\usepackage{verbatim,float,url,dsfont}
\usepackage{graphicx,subfigure,psfrag}
\usepackage{algorithm,algorithmic}
\usepackage{mathtools,enumitem}
\usepackage{microtype}
\usepackage{xr-hyper}
\usepackage[colorlinks=true,citecolor=blue,urlcolor=blue,linkcolor=blue]{hyperref}
\usepackage{multicol}

\newtheorem{theorem}{Theorem}
\newtheorem{lemma}{Lemma}
\newtheorem{corollary}{Corollary}
\newtheorem{proposition}{Proposition}
\theoremstyle{definition}
\newtheorem{remark}{Remark}

\def\R{\mathbb{R}}
\def\E{\mathbb{E}}
\def\P{\mathbb{P}}
\def\Q{\mathbb{Q}}
\def\Cov{\mathrm{Cov}}

\def\tf{\tilde{f}}

\def\one{\mathds{1}}
\def\cF{\mathcal{F}}
\def\cG{\mathcal{G}}

\def\TV{\mathrm{TV}}
\def\G{\mathbb{G}}
\def\ipm{\rho}

\begin{document}

\twocolumn[

\aistatstitle{A Higher-Order Kolmogorov-Smirnov Test}

\aistatsauthor{Veeranjaneyulu Sadhanala$^1$ \quad
  Yu-Xiang Wang$^2$ \quad Aaditya Ramdas$^1$ \quad
  Ryan J. Tibshirani$^1$}

\aistatsaddress{$^1$Carnegie Mellon University \And $^2$University of
  California at Santa Barbara}]

\runningauthor{Veeranjaneyulu Sadhanala, Yu-Xiang Wang, Aaditya Ramdas, Ryan
  J. Tibshirani} 

\begin{abstract} 
We present an extension of the Kolmogorov-Smirnov (KS) two-sample test, which
can be more sensitive to differences in the tails.  Our test statistic is an
integral probability metric (IPM) defined over a higher-order total variation ball,
recovering the original KS test as its simplest case.  We give an exact
representer result for our IPM, which generalizes the fact that the original
KS test statistic can be expressed in equivalent variational and CDF forms.  For
small enough orders ($k \leq 5$), we develop a linear-time algorithm for
computing our higher-order KS test statistic; for all others ($k \geq 6$), we
give a nearly linear-time approximation.  We derive the asymptotic null
distribution for our test, and show that our nearly linear-time approximation
shares the same asymptotic null.  Lastly, we complement our theory with
numerical studies. 
\end{abstract}

\section{INTRODUCTION}
\label{sec:intro}
The Kolmogorov-Smirnov (KS) test \citep{kolmogorov1933sulla,smirnov1948table} is
a classical and celebrated tool for nonparametric hypothesis testing.  Let
\smash{$x_1,\ldots,x_m \sim P$} and 
\smash{$y_1,\ldots,y_n \sim Q$} be independent samples.  Let \smash{$X_{(m)}$}
and \smash{$Y_{(n)}$} denote the two sets of samples, and also let
\smash{$Z_{(N)} = X_{(m)} \cup Y_{(n)} = \{z_1,\ldots,z_N\}$}, where $N=m+n$.
The two-sample KS test statistic is defined as
\begin{equation}
\label{eq:ks}
\max_{z \in Z_{(m+n)}} \; \bigg|\frac{1}{m}\sum_{i=1}^m 1\{x_i \leq z\} -
  \frac{1}{n}\sum_{i=1}^n 1\{y_i \leq z\} \bigg|.
\end{equation}
In words, this measures the maximum absolute difference between the empirical
cumulative distribution functions (CDFs) of \smash{$X_{(m)}$} and
\smash{$Y_{(n)}$}, across all points in the joint sample
\smash{$Z_{(m+n)}$}. Naturally, the two-sample KS test rejects the null
hypothesis of $P=Q$ for large values of the statistic. The statistic
\eqref{eq:ks} can also be written in the following variational form:
\begin{equation}
\label{eq:ks_var}
\sup_{f \,:\, \TV(f) \leq 1} \; |\P_m f - \Q_n f|,
\end{equation}
where $\TV(\cdot)$ denotes total variation, and we define the empirical
expectation operators $\P_m,\Q_n$ via 
$$
\P_m f = \frac{1}{m}\sum_{i=1}^m f(x_i) 
\;\; \text{and} \;\;
\Q_n f = \frac{1}{n} \sum_{i=1}^n f(y_i).
$$  
Later, we will give a general representation result that implies the 
equivalence of \eqref{eq:ks} and \eqref{eq:ks_var} as a special case.   

The KS test is a fast, general-purpose two-sample nonparametric test.
But being a general-purpose test also means that it is systematically less
sensitive to some types of differences, such as tail differences 
\citep{bryson1974heavy}.  Intuitively, this is because the empirical CDFs of
\smash{$X_{(m)}$} and \smash{$Y_{(n)}$} must both tend to 0 as $z \to -\infty$
and to 1 as $z \to \infty$, so the gap in the tails will not be large.

The insensitivity of the KS test to tail differences is well-known.  Several
authors have proposed modifications to the KS test to improve its tail
sensitivity, based on variance-reweighting \citep{anderson1952asymptotic}, or
Renyi-type statistics \citep{mason1983modified,calitz1987alternative}, to name a
few ideas.  In a different vein, \citet{wang2014falling} recently proposed a
higher-order extension of the KS two-sample test, which replaces the total
variation constraint on $f$ in \eqref{eq:ks_var} with a total variation
constraint on a derivative of $f$.  These authors show empirically that, in some
cases, this modification can lead to better tail sensitivity. In the current
work, we refine the proposal of \citet{wang2014falling}, and give theoretical
backing for this new test.   

\paragraph{A Higher-Order KS Test.} Our test statistic has the form of an
integral probability metric (IPM). For a function class $\cF$, the IPM between
distributions $P$ and $Q$, with respect to $\cF$, is defined as
\citep{muller1997integral} 
\begin{equation}
\label{eq:ipm}
\ipm(P,Q; \cF) = \sup_{f\in \cF} \; | \P f  - \Q f |
\end{equation}
where we define the expectation operators $\P,\Q$ by 
$$
\P f = \E_{X \sim P}[f(X)] \;\; \text{and} \;\; \Q f=\E_{Y \sim Q}[f(Y)]. 
$$
For a given function class $\cF$, the IPM $\ipm(\cdot,\cdot\,; \cF)$ is a 
pseudometric on the space of distributions. Note that the KS test in
\eqref{eq:ks_var} is precisely $\ipm(P_m,Q_n; \cF_0)$, where $P_m,Q_n$ are the
empirical distributions of \smash{$X_{(m)},Y_{(n)}$}, respectively, and $\cF_0 =
\{ f : \TV(f) \leq 1\}$.

Consider an IPM given by replacing $\cF_0$ with $\cF_k=\{f : 
\TV(f^{(k)}) \leq 1\}$, for an integer $k \geq 1$ (where we write $f^{(k)}$ for
the  $k$th weak derivative of $f$).  Some motivation is as follows.  In the case 
$k=0$, we know that the {\it witness functions} in the KS test
\eqref{eq:ks_var}, i.e., the functions in $\cF_0$ that achieve the
supremum, are piecewise constant step functions (cf.\ the equivalent
representation \eqref{eq:ks}).  These functions can only have so much action in
the tails.  By moving to $\cF_k$, which is essentially comprised of the $k$th
order antiderivative of functions in $\cF_0$, we should expect that the witness 
functions over $\cF_k$ are $k$th order antiderivatives of piecewise constant
functions, i.e., $k$th degree piecewise polynomial functions, which can have
much more sensitivity in the tails. 

But simply replacing $\cF_0$ by $\cF_k$ and proposing to compute 
$\ipm(P_m,Q_n; \cF_k)$ leads to an ill-defined test.  This is due to the fact
that $\cF_k$ contains all polynomials of degree $k$.  Hence, if the 
$i$th moments of $P_m,Q_n$ differ, for any $i \in [k]$ (where we abbreviate
$[a]=\{1,\ldots,a\}$ for an integer $a \geq 1$), then $\ipm(P_m,Q_n;
\cF_k)=\infty$.    

As such, we must modify $\cF_k$ to control the growth of its elements.
While there are different ways to do this, not all result in computable
IPMs. The approach we take yields an exact representer theorem (generalizing the
equivalence between \eqref{eq:ks} and \eqref{eq:ks_var}). Define   
\begin{multline}
\label{eq:fk}
\cF_k = \big\{ f : \TV(f^{(k)}) \leq 1, \\
 f^{(j)}(0) =0, \; j \in \{0\} \cup [k-1], \\  
f^{(k)}(0+)=0 \; \text{or} \; f^{(k)}(0-)=0 \big\}.
\end{multline}
Here \smash{$f^{(k)}(0+)$} and \smash{$f^{(k)}(0-)$} denote one-sided limits 
at 0 from above and below, respectively.  Informally, the functions in $\cF_k$
are pinned down at 0, with all lower-order derivatives (and the limiting
$k$th derivative from the right or left) equal to 0, which limits their
growth. Now we define the {\it $k$th-order KS test statistic} as   
\begin{equation}
\label{eq:hks}
\ipm(P_m,Q_n;\cF_k) = \sup_{f \in \cF_k} \; |\P_m f - \Q_n f|.  
\end{equation} 
An important remark is that for $k=0$, this recovers the original KS test 
statistic \eqref{eq:ks_var}, because $\cF_0$ contains all step functions of the
form $g_t(x) = 1\{x \leq t\}$, $t \geq 0$.  

Another important remark is that for any $k \geq 0$, the function class $\cF_k$
in \eqref{eq:fk} is ``rich enough'' to make the IPM in \eqref{eq:hks} a
metric. We state this formally next; its proof, as with all other proofs, is
in the appendix.

\begin{proposition}
\label{prop:metric}
For any $k\geq 0$, and any $P,Q$ with $k$ moments, $\ipm(P,Q;\cF_k) = 0$
if and only if $P=Q$.  
\end{proposition}

\paragraph{Motivating Example.} 
Figure \ref{fig:motivation1} shows the results of a simple simulation comparing
the proposed higher-order tests \eqref{eq:hks}, of orders $k=1$ through 5,
against the usual KS test (corresponding to $k=0$).  For the simulation setup,
we used $P=N(0,1)$ and $Q=N(0,1.44)$. For 500 ``alternative'' repetitions, we
drew $m=250$ samples from $P$, drew $n=250$ samples from $Q$, and computed 
test statistics; for another 500 ``null'' repetitions, we permuted the $m+n=500$ 
samples from the corresponding alternative repetition, and again computed test
statistics.  For each test, we varied the rejection threshold for each test, we
calculated its true positive rate
using the alternative repetitions, and calculated its false positive rate using
the null repetitions.  The oracle ROC curve corresponds to the likelihood ratio
test (which knows the exact distributions $P,Q$).  Interestingly, we can see
that power of the higher-order KS test improves as we increase the order from
$k=0$ up to $k=2$, then stops improving by $k=3,4,5$.

\begin{figure}[htb]
\includegraphics[width=\columnwidth]{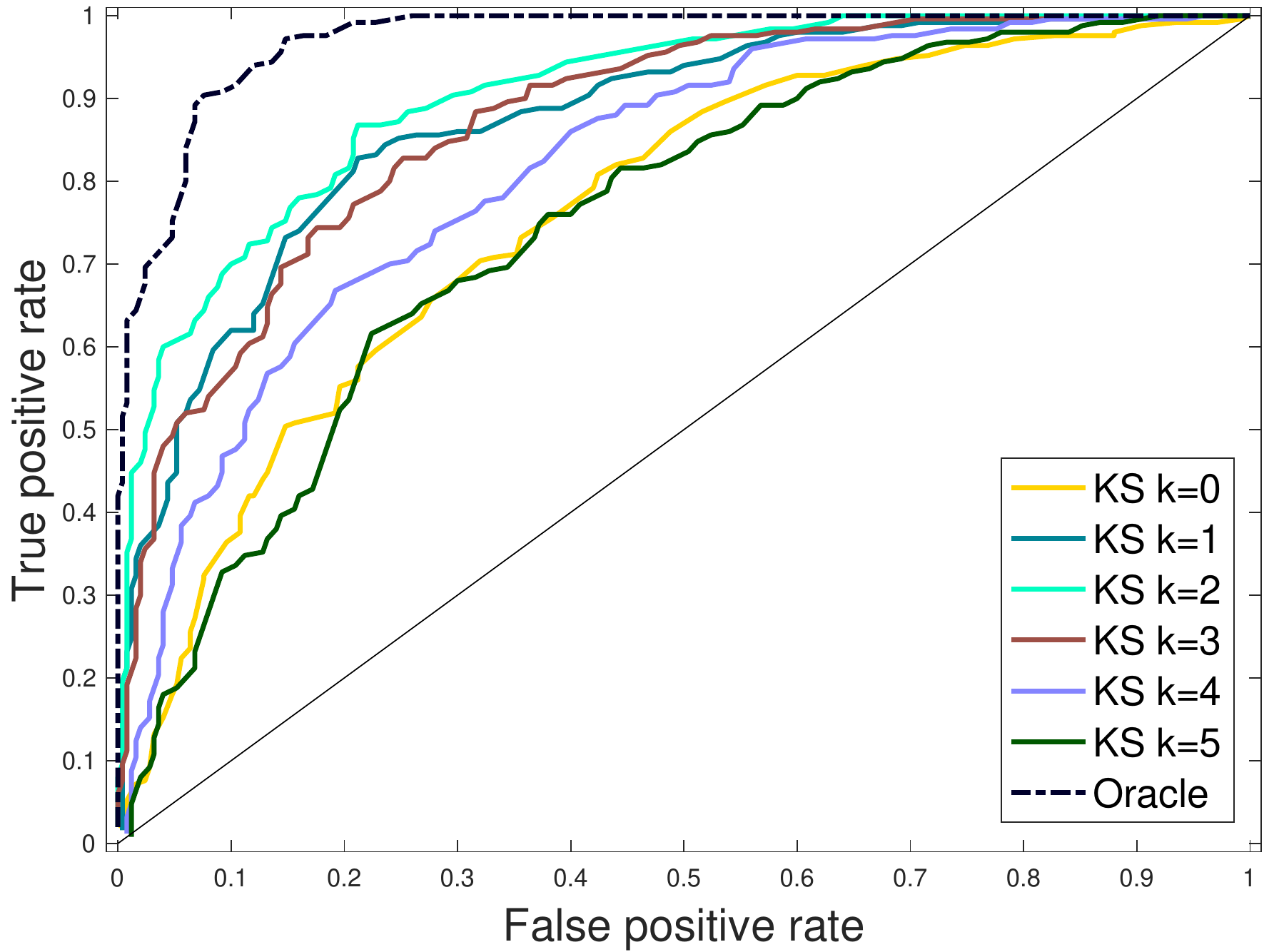}   
\caption{\it\small ROC curves from an experiment comparing the proposed
  higher-order KS tests in \eqref{eq:hks} (for various $k$) to the usual KS
  test, when $P=N(0,1)$ and $Q=N(0,1.44)$.}
\label{fig:motivation1}
\end{figure}

Figure \ref{fig:motivation2} displays the witness function (which achieves the
supremum in \eqref{eq:hks}) for a large-sample version of the higher-order KS
test, across orders $k=0$ through 5. We used the same distributions
as in Figure \ref{fig:motivation1}, but now $n=m=10^4$.  We will prove in
Section \ref{sec:computation} that, for the $k$th order test, the witness
function is always a $k$th degree piecewise polynomial (in fact, a rather
simple one, of the form \smash{$g_t(x) = (x-t)_+^k$} or \smash{$g_t(x) =
  (t-x)_+^k$} for a knot $t$).  Recall the underlying distributions $P,Q$ here
have different variances, and we can see from their witness functions that all   
higher-order KS tests choose to put weight on tail differences.  Of
course, the power of any test of is determined by the size of the statistic
under the alternative, relative to typical fluctuations under the null.  As we
place more weight on tails, in this particular setting, we see diminishing
returns at $k=3$, meaning the null fluctuations must be too great.

\begin{figure}[htb]
\includegraphics[width=\columnwidth]{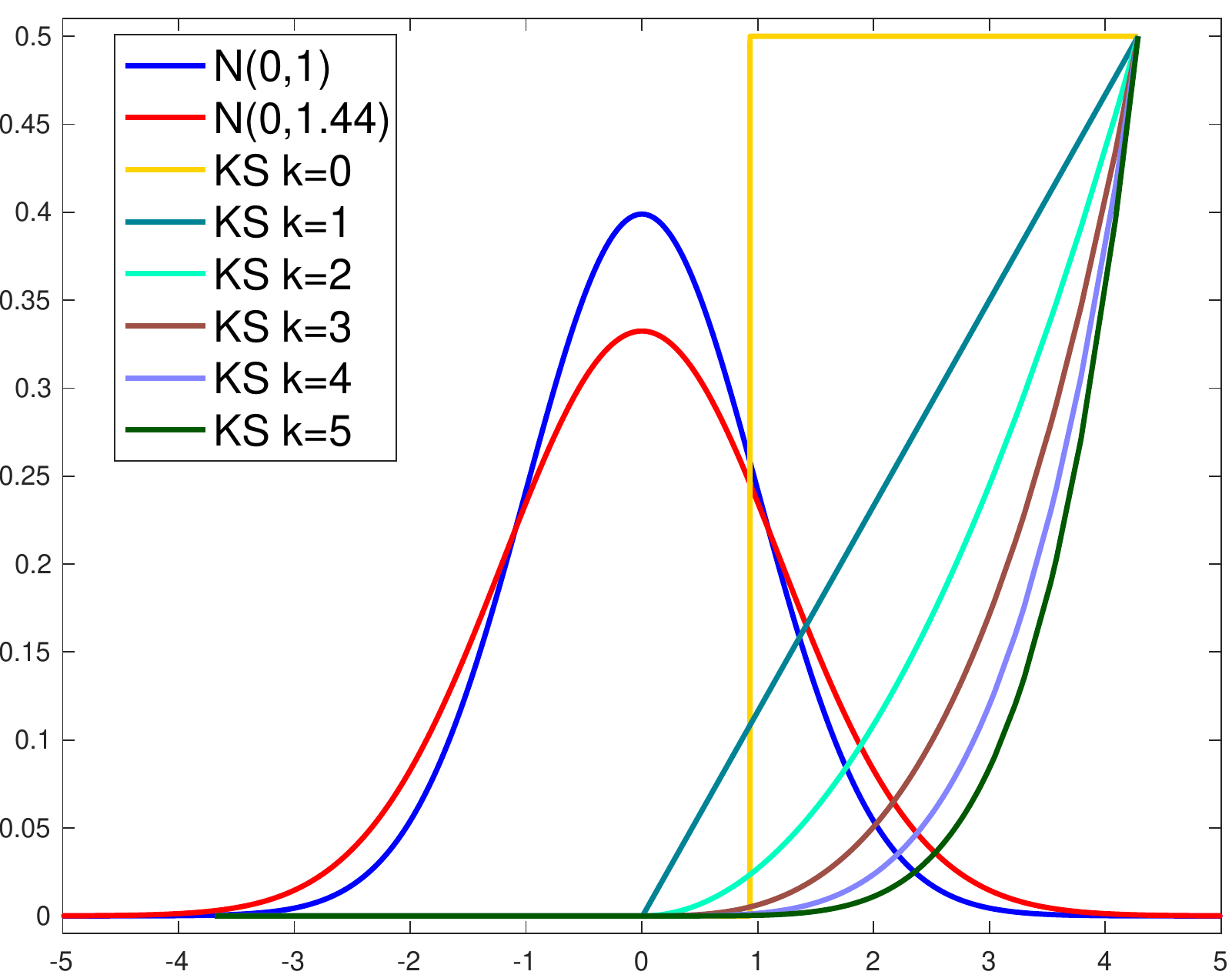}
\caption{\it\small Witness functions (normalized for plotting purposes) for the
  higher-order KS tests, when $P=N(0,1)$ and $Q=N(0,1.44)$. They are always of
  piecewise  polynomial form; and here they all place weight on tail
  differences.} 
\label{fig:motivation2}
\end{figure}

\paragraph{Summary of Contributions.} Our contributions in this work are as
follows.    

\begin{itemize}
\item We develop an exact representer theorem for the higher-order KS test
  statistic \eqref{eq:hks}.  This enables us to compute the test statistic in
  linear-time, for all $k \leq 5$.  For $k \geq 6$, we develop a nearly
  linear-time approximation to the test statistic.

\item We derive the asymptotic null distribution of the our higher-order KS test
  statistic, based on empirical process theory.  For $k \geq 6$, our 
  approximation to the test statistic has the same asymptotic null. 

\item We provide concentration tail bounds for the test statistic.  Combined
  with the metric property from Proposition \ref{prop:metric}, this shows that
  our higher-order KS test is asymptotically powerful against any pair of fixed,
  distinct distributions $P,Q$. 

\item We perform extensive numerical studies to compare the newly proposed tests 
  with several others.
\end{itemize}

\paragraph{Other Related Work.}  Recently, IPMs have been gaining in popularity
due in large part to energy distance tests
\citep{szekely2004testing,baringhaus2004new} and kernel maximum mean discrepancy
(MMD) tests \citep{gretton2012kernel}, and in fact, there is an equivalence
between the two classes \citep{sejdinovic2013equivalence}.  An IPM 
with a judicious choice of $\cF$ gives rise to a number of common distances
between distributions, such as Wasserstein distance or total variation (TV)
distance. While IPMs look at differences $dP-dQ$, tests based on
$\phi$-divergences (such as Kullback-Leibler, or Hellinger) look at ratios
$dP/dQ$, but can be hard to efficiently estimate in practice
\citep{sriperumbudur2009integral}.  The TV distance is the only IPM 
that is also a $\phi$-divergence, but it is impossible to estimate. 

There is also a rich class of nonparametric tests based on graphs.
Using minimum spanning trees, \citet{friedman1979multivariate} generalized both
the Wald-Wolfowitz runs test and the KS test. Other tests are based on k-nearest 
neighbors graphs \citep{schilling1986multivariate,henze1988multivariate} or
matchings \citep{rosenbaum2005exact}.  The Mann-Whitney-Wilcoxon test has a
multivariate generalization using the concept of data depth
\citep{liu1993quality}. \citet{bhattacharya2016power} established that many 
computationally efficient graph-based tests have suboptimal statistical
power, but some inefficient ones have optimal scalings. 

Different computational-statistical tradeoffs were also discovered for
IPMs \citep{ramdas2015adaptivity}.  Further, as noted by
\citet{janssen2000global} (in the context of one-sample testing), every
nonparametric test is essentially powerless in an infinity of directions, and
has nontrivial power only against a finite subspace of alternatives. In
particular, this implies that no single nonparametric test can uniformly
dominate all others; improved power in some directions generally implies weaker
power in others. This problem only gets worse in high-dimensional settings 
\citep{ramdas2015decreasing,arias2018remember}. Therefore, the question of which 
test to use for a given problem must be guided by a combination of
simulations, computational considerations, a theoretical understanding of the
pros/cons of each test, and a practical understanding of the data at hand. 

\paragraph{Outline.} In Section \ref{sec:computation}, we give computational
details for the higher-order KS test statistic \eqref{eq:hks}.  We derive its
asymptotic null in Section \ref{sec:null}, and give concentration bounds (for
the statistic around the population-level IPM) in Section \ref{sec:concentration}.
We give numerical experiments in Section \ref{sec:experiments}, and conclude in 
Section \ref{sec:discussion} with a discussion. 

\section{COMPUTATION} 
\label{sec:computation}

Write $T=\ipm(P_m,Q_n;\cF_k)$ for the test statistic in \eqref{eq:hks}.  In this
section, we derive a representer theorem for $T$, develop a linear-time
algorithm for $k \leq 5$, and a nearly linear-time approximation for $k \geq 6$.     

\subsection{Representer Theorem}

The higher-order KS test statistic in \eqref{eq:hks} is defined by an
infinite-dimensional maximization over $\cF_k$ in \eqref{eq:fk}.  Fortunately,
we can restrict our attention to a simpler function class, as we show next.

\begin{theorem}
\label{thm:hks_representer}
Fix $k \geq 0$.  Let \smash{$g^+_t(x) = (x-t)_+^k/k!$} and
\smash{$g^-_t(x) = (t-x)_+^k/k!$} for $t \in \R$, where we write
$(a)_+ = \max\{a,0\}$. For the statistic $T$ defined by \eqref{eq:hks},   
\begin{multline}
\label{eq:hks_representer}
T = \max\Big\{ \sup_{t \geq 0} \; |(\P_m - \Q_n) g^+_t|, \\
\sup_{t \leq 0} \; |(\P_m - \Q_n) g^-_t| \Big\}.
\end{multline}
\end{theorem}

The proof of this theorem uses a key result from 
\citet{mammen1991nonparametric}, where it is shown that we can construct a 
spline interpolant to a given function at given points, such that its
higher-order total variation is no larger than that of the original function.  

\begin{remark}
When $k=0$, note that for $t \geq 0$,
\begin{align*}
|(\P_m - \Q_n) g^+_t| 
&= \bigg|\frac{1}{m}\sum_{i=1}^m 1\{x_i > t\} -
  \frac{1}{n}\sum_{i=1}^n 1\{y_i > t\} \bigg| \\
&= \bigg|\frac{1}{m}\sum_{i=1}^m 1\{x_i \leq t\} -
  \frac{1}{n}\sum_{i=1}^n 1\{y_i \leq t\} \bigg|
\end{align*}
and similarly for $t \leq 0$, $|(\P_m - \Q_n) g^-_t|$ reduces to the same
expression in the second line above.  As we vary $t$ from $-\infty$ to 
$\infty$, this only changes at values \smash{$t \in
  Z_{(N)}$}, which shows \eqref{eq:hks_representer} and \eqref{eq:ks}
are the same, i.e., Theorem \ref{thm:hks_representer} recovers the 
equivalence between \eqref{eq:ks_var} and \eqref{eq:ks}.
\end{remark} 

\begin{remark}
For general $k \geq 0$, we can interpret \eqref{eq:hks_representer} as a
comparison between truncated $k$th order moments, between the empirical
distributions $P_m$ and $Q_n$.  The test statistic $T$ the maximum over all
possible truncation locations $t$.  The critical aspect here is {\it truncation},
which makes the higher-order KS test statistic a metric (recall Proposition
\ref{prop:metric}).  A comparison of moments, alone, would 
not be enough to ensure such a property.
\end{remark}

Theorem \ref{thm:hks_representer} itself does not immediately lead
to an algorithm for computing $T$, as the range of $t$ considered in the
suprema is infinite. However, through a bit more work, detailed in the next two
subsections, we can obtain an exact linear-time algorithm for all $k \leq 5$,  
and a linear-time approximation for $k \geq 6$.

\subsection{Linear-Time Algorithm for $k \leq 5$} 

The key fact that we will exploit is that the criterion in
\eqref{eq:hks_representer}, as a function of $t$, is a piecewise polynomial of
order $k$ with knots in \smash{$Z_{(N)}$}. Assume without a loss of generality  
that $z_1 < \cdots < z_N$.  Also assume without a loss of generality that $z_1
\geq 0$ (this simplifies notation, and the general case follows
by the repeating the same arguments separately for the points in
\smash{$Z_{(N)}$} on either side of 0).  
Define $c_i=\one\{z_i\in X_{(m)}\}/m -\one\{z_i\in Y_{(n)}\}/n$, $i\in [N]$, and   
\begin{equation}
\label{eq:phi}
\phi_i(t) = \frac{1}{k!} \sum_{j=i}^N c_j (z_j - t)^k, \;\; i \in [N]. 
\end{equation}
Then the statistic in \eqref{eq:hks_representer} can be succinctly written as
\begin{equation}
\label{eq:hks_exact_phi}
T = \max_{i \in [N]} \; \sup_{t \in [z_{i-1},z_i]} \; \phi_i(t),
\end{equation}
where we let $z_0=0$ for convenience.  Note each $\phi_i(t)$, $i \in [N]$
is a $k$th degree polynomial. We can compute a representation for these
polynomials efficiently.      

\begin{lemma}
\label{lem:phi_recurrence}
Fix $k \geq 0$. The polynomials in \eqref{eq:phi} satisfy the recurrence
relations
$$
\phi_i(t) = \frac{1}{k!}  c_i (z_i - t)^k + \phi_{i+1}(t), \;\; i\in [N]
$$
(where $\phi_{N+1}=0$).  Given the monomial expansion
$$
\vspace{-2pt}
\phi_{i+1}(t) = \sum_{\ell=0}^k a_{i+1,\ell} t^\ell,
\vspace{-2pt}
$$
we can compute an expansion for $\phi_i$, with coefficients $a_{i\ell}$, $\ell  
\in \{0\}\cup[k]$, in $O(1)$ time.  So we can compute all coefficients 
$a_{i,\ell}$, $i \in [N]$, $\ell \in \{0\}\cup[k]$ in $O(N)$ time. 
\end{lemma}

To compute $T$ in \eqref{eq:hks_exact_phi}, we must maximize
each polynomial $\phi_i$ over its domain $[z_{i-1},z_i]$, for $i \in [N]$, and then
compare maxima. Once we have computed a representation for these polynomials, 
as Lemma \ref{lem:phi_recurrence} ensures we can do in $O(N)$ time, we can use
this to analytically maximize each polynomial over its domain, provided the 
order $k$ is small enough.  Of course, maximizing a polynomial over an interval
can be reduced to computing the roots of its derivative, which is an analytic
computation for any $k \leq 5$ (since the roots of any quartic have a
closed-form, see, e.g., \citealt{rosen1995niels}). The next result summarizes.  

\begin{proposition}
\label{prop:hks_exact_k5}
For any $0 \leq k \leq 5$, the test statistic in \eqref{eq:hks_exact_phi}
can be computed in $O(N)$ time. 
\end{proposition}

Maximizing a polynomial of degree $k \geq 6$ is not generally possible in 
closed-form. However, developments in semidefinite optimization allow us to  
approximate its maximum efficiently, investigated next. 

\subsection{Linear-Time Approximation for $k \geq 6$}

Seminal work of \citet{shor1998nondifferentiable,nesterov2000squared} shows that
the problem of maximizing a polynomial over an interval can be cast as a 
semidefinite program (SDP). The number of variables in this SDP depends only on
the polynomial order $k$, and all constraint functions are self-concordant.  
Using say an interior point method to solve this SDP, therefore, leads to the 
following result.  

\begin{proposition}
\label{prop:hks_approx_k6}
Fix $k \geq 6$ and $\epsilon > 0$. For each polynomial in \eqref{eq:phi}, we can 
compute an $\epsilon$-approximation to its maximum in $c_k \log(1/\epsilon)$  
time, for a constant $c_k>0$ depending only on $k$.  As we can compute a
representation for all these polynomials in $O(N)$ time
(Lemma~\ref{lem:phi_recurrence}), this means we can compute an
$\epsilon$-approximation to the statistic in \eqref{eq:hks_representer} in $O(N
\log(1/\epsilon))$ time.  
\end{proposition}

\begin{remark}
Let $T_\epsilon$ denote the $\epsilon$-approximation from Proposition
\ref{prop:hks_approx_k6}. Under the null $P=Q$, we would need to have 
\smash{$\epsilon=o(1/\sqrt{N})$} in order for the approximation $T_\epsilon$
to share the asymptotic null distribution of $T$, as we will see in Section
\ref{sec:null_hks_approximation}. Taking say, $\epsilon=1/N$, the statistic
$T_{1/N}$ requires $O(N\log{N})$ computational time, and this is why in various  
places we make reference to a {\it nearly} linear-time approximation when $k
\geq 6$. 
\end{remark}

\subsection{Simple Linear-Time Approximation}

We conclude this section by noting a simple approximation to
\eqref{eq:hks_representer}  given by  
\begin{multline}
\label{eq:hks_approx_simple}
T^* = \max\Big\{ \max_{t \in Z^0_{(N)}, \, t \geq 0} \; 
|(\P_m - \Q_n) g^+_t|, \\   
\max_{t \in  Z^0_{(N)}, \, t \leq 0 } \; 
|(\P_m - \Q_n) g^-_t| \Big\},
\end{multline}
where \smash{$Z^0_{(N)} = \{0\} \cup Z_{(N)}$}. Clearly, for $k=0$ or 1, the
maximizing $t$ in \eqref{eq:hks_representer} must be one of the sample
points \smash{$Z_{(N)}$}, so $T^*=T$ and there is no approximation error in   
\eqref{eq:hks_approx_simple}.  For $k \geq 2$, we can control the error as
follows.  

\begin{lemma}
\label{lem:hks_approx_simple}
For $k \geq 2$, the statistics in \eqref{eq:hks_representer},
\eqref{eq:hks_approx_simple} satisfy 
$$
T - T^* = \frac{\delta_N}{(k-1)!}  \bigg( \frac{1}{m}\sum_{i=1}^m |x_i|^{k-1}
+\frac{1}{n}\sum_{i=1}^n |y_i|^{k-1} \bigg),
$$
where $\delta_N$ is the maximum gap between sorted points in
\smash{$Z^0_{(N)}$}.    
\end{lemma}

\begin{remark}
We would need to have \smash{$\delta_N=o_P(1/\sqrt{N})$}
in order for $T^*$ to share the asymptotic null of $T$, see again Section
\ref{sec:null_hks_approximation} (this is assuming that $P$ has $k-1$ moments,
so the sample moments concentrate for large enough $N$).  This
will not be true of $\delta_N$, the maximum gap, in general.  But it does hold
when $P$ is continuous, having compact support, and a density bounded from below 
on its support; here, in fact, $\delta_N=o_P(\log{N}/N)$ (see, e.g., 
\citealt{wang2014falling}).    
\end{remark}

Although it does not have the strong guarantees of the approximation from 
Proposition \ref{prop:hks_approx_k6}, the statistic in
\eqref{eq:hks_approx_simple} is simple and efficient---we must emphasize that it
can be computed in $O(N)$ linear time, as a consequence of Lemma 
\ref{lem:phi_recurrence} (the evaluations of $\phi_i(t)$ at the sample points
\smash{$t \in Z_{(N)}$} are the constant terms $a_{i0}$, $i \in [N]$ in their
monomial expansions)---and is likely a good choice for most practical purposes.   


\section{ASYMPTOTIC NULL} 
\label{sec:null}

To study the asymptotic null distribution of the proposed higher-order KS test,
we will appeal to uniform central limit theorems (CLTs) from the empirical
process theory literature, reviewed here for completeness.  
For functions $f,g$ in a class $\cF$, let $\G_{P,\cF}$ denote a Gaussian process
indexed by $\cF$ with mean and covariance
\begin{gather*}
 \E( \G_{P,\cF} f ) = 0, \;\; f \in \cF, \\
\Cov ( \G_{P,\cF} f, \, \G_{P,\cF} g ) =
\Cov_{X \sim P}(f(X), g(X)), \;\; f,g \in \cF. 
\end{gather*}
For functions $l,u$, let $[l,u]$ denote the set of functions $\{ f : l(x) \leq
f(x) \leq u(x), \; \text{for all $x$} \}$.  Call $[l,u]$ a {\it bracket} of size 
$\| u-l \|_2$, where $\|\cdot\|_2$ denotes the $L_2(P)$ norm, defined as 
$$
\|f\|_2^2 = \int f(x)^2 \, dP(x).
$$
Finally, let \smash{$N_{[]}(\epsilon, \|\cdot\|_2, \cF)$} be the smallest number 
of $\epsilon$-sized brackets that are required to cover $\cF$. Define the
bracketing integral of $\cF$ as 
$$
J_{[]}(\|\cdot\|_2, \cF) = \int_0^1
 \sqrt{\log N_{[]} (\epsilon, \|\cdot\|_2, \cF)} \, d\epsilon.
$$
Note that this is finite when \smash{$\log N_{[]}(\epsilon, \|\cdot\|_2,
  \cF)$} grows slower than $1/\epsilon^2$. We now state an important
uniform CLT from empirical process theory.  

\begin{theorem}[Theorem 11.1.1 in \citealt{dudley1999uniform}]
\label{thm:null_gen}
If $\cF$ is a class of functions with finite bracketing integral,
then when $P=Q$ and $m,n \to\infty$, the process 
$$
\sqrt{\frac{mn}{m+n}} \{ \P_m f - \Q_n f \}_{f \in \cF}
$$
converges weakly to the Gaussian process \smash{$\G_{P,\cF}$}. Hence,
$$
\sqrt{\frac{mn}{m+n}} \sup_{f\in \cF} \; |\P_m f - \Q_n f |
\overset{d}{\to} \sup_{f\in \cF} \; |\G_{P,\cF} f|. 
$$
\end{theorem}

\subsection{Bracketing Integral Calculation}
\label{sec:bracket}

To derive the asymptotic null of the higher-order KS test, based on its
formulation in \eqref{eq:hks}, and Theorem \ref{thm:null_gen}, we would need to
bound the bracketing integral of $\cF_k$.  While there are well-known entropy
(log covering) number bounds for related function classes (e.g.,
\citealt{birman1961piecewise,babenko1979theoretical}), and the conversion from
covering to bracketing numbers is standard, these results unfortunately require
the function class to be uniformly bounded in the sup norm, which is certainly
not true of $\cF_k$.

Note that the representer result in \eqref{eq:hks_representer} can be written as 
$T=\rho(P_m,Q_n;\cG_k)$, where 
\begin{equation}
\label{eq:gk} 
\cG_k = \{ g^+_t : t \geq 0 \} \cup \{ g^-_t : t \leq 0 \}.
\end{equation} 
We can hence instead apply Theorem \ref{thm:null_gen} to $\cG_k$, whose
bracketing number can be bounded by direct calculation, assuming enough moments
on $P$. 

\begin{lemma}
\label{lem:gk_bracket}
Fix $k \geq 0$. Assume {$\E_{X \sim P} |X|^{2k+\delta} \leq M < \infty$}, for
some $\delta > 0$. For the class $\cG_k$ in \eqref{eq:gk}, there is a constant
$C>0$ depending only on $k,\delta$ such that  
$$
 \log N_{[]}(\epsilon, \|\cdot\|_2, \cG_k) \leq 
C \log \frac{M^{1+ \frac{\delta(k-1)}{2k+\delta}}}{\epsilon^{2+\delta}}.
$$
\end{lemma}

\subsection{Asymptotic Null for Higher-Order KS}
\label{sec:null_hks}

Applying Theorem \ref{thm:null_gen} and Lemma \ref{lem:gk_bracket} to the
higher-order KS test statistic \eqref{eq:hks_representer} leads to the following
result. 

\begin{theorem}
\label{thm:null_hks}
Fix $k \geq 0$. Assume {$\E_{X \sim P} |X|^{2k+\delta} < \infty$}, for some
$\delta > 0$. When $P=Q$, the test statistic in \eqref{eq:hks_representer} 
satisfies, as $m,n \to \infty$,  
$$
\sqrt{\frac{mn}{m+n}} T \overset{d}{\to} \sup_{g \in \cG_k} \; | \G_{P,k} g |,  
$$
where \smash{$\G_{P,k}$} is an abbreviation for the Gaussian process indexed  
by the function class $\cG_k$ in \eqref{eq:gk}.
\end{theorem}

\begin{remark}
When $k=0$, note that for $t \geq s \geq 0$, the covariance function is  
$$
\Cov_{X \sim P}(1\{X>s\}, 1\{X>t\}) = F_P(s)(1-F_P(t)), 
$$
where $F_P$ denotes the CDF of $P$. For $s \leq t \leq 0$, the covariance
function is again equal to \smash{$F_P(s)(1-F_P(t))$}.  The supremum of this
Gaussian process over $t \in \R$ is that of a Brownian bridge, so Theorem 
\ref{thm:null_hks} recovers the well-known asymptotic null distribution of the
KS test, which (remarkably) does not depend on $P$.  
\end{remark}

\begin{remark}
When $k \geq 1$, it is not clear how strongly the supremum of the Gaussian
process from Theorem \ref{thm:null_hks} depends on $P$; it appears it must
depend on the first $k$ moments of $P$, but is not clear whether it {\it only}
depends on these moments.  Section \ref{sec:experiments} investigates
empirically.  Currently, we do not have a precise understanding of whether the  
asymptotic null is useable in practice, and we suggest using a permutation null 
instead. 
\end{remark}

\subsection{Asymptotic Null Under Approximation}
\label{sec:null_hks_approximation}

The approximation from Proposition \ref{prop:hks_approx_k6} shares the same 
asymptotic null, provided $\epsilon>0$ is small enough.

\begin{corollary}
\label{cor:null_hks_approx_k6}
Fix $k \geq 0$. Assume {$\E_{X \sim P} |X|^{2k+\delta} < \infty$}, for some
$\delta > 0$. When $P=Q$, as $m,n \to \infty$ such that $m/n$ converges to a
positive constant, the test statistic $T_\epsilon$ from Proposition
\ref{prop:hks_approx_k6} converges at a \smash{$\sqrt{N}$}-rate to the supremum
of the same Gaussian process in Theorem \ref{thm:null_hks}, provided 
\smash{$\epsilon=o(1/\sqrt{N})$}.    
\end{corollary}

The approximation in \eqref{eq:hks_approx_simple} shares the same asymptotic
null, provided $P$ is continuous with compact support.

\begin{corollary}
\label{cor:null_hks_approx_simple}
Fix $k \geq 0$. Assume that $P$ is continuous, compactly supported, with density
bounded from below on its support.  When $P=Q$, as $m,n \to \infty$ such that 
$m/n$ converges to a positive constant, the test statistic $T^*$ in
\eqref{eq:hks_approx_simple} converges at a  \smash{$\sqrt{N}$}-rate to the
supremum of the same Gaussian process in Theorem \ref{thm:null_hks}.   
\end{corollary}

\section{TAIL CONCENTRATION}
\label{sec:concentration}

We examine the convergence of our test statistics to their population
analogs.  In general, if the population-level IPM $\ipm(P,Q;\cF_k)$ is large, 
then the concentration bounds below will imply that the empirical statistic 
$\ipm(P_m,Q_n;\cF_k)$ will be large for $m,n$ sufficiently large, and the
test will have power. 

We first review the necessary machinery, again from empirical process
theory. For $p \geq 1$, and a function $f$ of a random variable $X \sim
P$, recall the $L_p(P)$ norm is defined as $\|f\|_p = [\E(f(X)^p)]^{1/p}$.  For 
$p > 0$, recall the {\it exponential Orlicz norm} of order $p$ is defined as 
$$
\|f\|_{\Psi_p} = \inf \big\{ t > 0: \E [\exp(|X|^p/t^p)] -1 \leq 1\big\}.
$$
(These norms depend on the measure $P$, since they are defined in terms of 
expectations with respect to $X \sim P$, though this is not explicit in our
notation.)

We now state an important concentration result. 

\begin{theorem}[Theorems 2.14.2 and 2.14.5 in \citealt{vandervaart1996weak}]
\label{thm:concentration_gen} 
Let $\cF$ be a class functions with an envelope function $F$, i.e., $f \leq
F$ for all $f \in \cF$. Define
$$
W = \sqrt{n} \sup_{f \in \cF} \;| \P_n f - \P f |,
$$
and abbreviate \smash{$J=J_{[]}(\|\cdot\|,\cF)$}. For $p \geq 2$, if $\|F\|_p < 
\infty$, then for a constant $c_1>0$,  
$$
[\E(W^p)]^{1/p} \leq c_1 \Big(\|F\|_2 J + n^{-1/2+1/p} \|F\|_p\Big),
$$
and for $0 < p \leq 1$, if \smash{$\|F\|_{\Psi_p} < \infty$}, then for a
constant $c_2>0$, 
$$
\|W\|_{\Psi_p} \leq  c_2 \Big(\|F\|_2  J + n^{-1/2} (1+\log n)^{1/p}
\|F\|_{\Psi_p}\Big). 
$$
\end{theorem}

The two-sample test statistic $T=\ipm(P_m,Q_n;\cG_k)$ satisfies (following by a
simple argument using convexity)
$$
| T- \ipm(P,Q;\cF_k)| \leq \ipm(P,P_m;\cF_k) + \ipm(Q,Q_n;\cF_k). 
$$
The terms on the right hand side can each be bounded by Theorem
\ref{thm:concentration_gen}, where we can use the envelope function
$F(x)=|x|^k/k!$ for $\cG_k$.  Using Markov's inequality, we can then get a tail
bound on the statistic.   

\begin{theorem}
\label{thm:concentration_hks}
Fix $k \geq 0$. Assume that $P,Q$ both have $p$ moments, where $p \geq 2$ and 
$p > 2k$. For the statistic in \eqref{eq:hks_representer}, for any $\alpha > 0$,
with probability $1-\alpha$,   
$$
|T- \ipm(P,Q; \cG_k)| \leq c(\alpha) \bigg(\frac{1}{\sqrt{m}} +
\frac{1}{\sqrt{n}} \bigg), 
$$ 
where \smash{$c(\alpha)=c_0 \alpha^{-1/p}$}, and $c_0>0$ is a constant.  If $P,Q$
both have finite exponential Orlicz norms of order $0 < p \leq 1$, then the
above holds for \smash{$c(\alpha)=c_0(\log(1/\alpha))^{1/p}$}.    
\end{theorem}

When we assume $k$ moments, the population IPM for $\cF_k$ also has a 
representer in $\cG_k$; by Proposition \ref{prop:metric}, this implies
$\ipm(\cdot,\cdot\,;\cG_k)$ is also a metric.    

\begin{corollary}
\label{cor:hks_representer_pop}
Fix $k \geq 0$. Assuming $P,Q$ both have $k$ moments, $\ipm(P,Q;\cF_k)
= \ipm(P,Q;\cG_k)$.  Therefore, by Proposition \ref{prop:metric},
$\ipm(\cdot,\cdot\,;\cG_k)$ is a metric (over the space of distributions $P,Q$  
with $k$ moments). 
\end{corollary}

Putting this metric property together with Theorem \ref{thm:concentration_hks}
gives the following. 

\begin{corollary}
\label{cor:powerful}
Fix $k \geq 0$.  For $\alpha_N = o(1)$ and $1/\alpha_N=o(N^{p/2})$,
reject when the higher-order KS test statistic \eqref{eq:hks_representer}
satisfies \smash{$T > c(\alpha_N)(1/\sqrt{m} + 1/\sqrt{n})$}, 
where $c(\cdot)$ is as in Theorem \ref{thm:concentration_hks}.  For any
$P,Q$ that meet the moment conditions of Theorem \ref{thm:concentration_hks},
as $m,n \to \infty$ in such a way that $m/n$ approaches a positive constant, we
have type I error tending to 0, and power tending to 1, i.e., the higher-order
KS test is {\em asymptotically powerful}.    
\end{corollary}

\section{NUMERICAL EXPERIMENTS}
\label{sec:experiments}

We present numerical experiments that examine the convergence of our test 
statistic to its asymptotic null, its power relative to other general purpose
nonparametric tests, and its power when $P,Q$ have densities with local
differences. Experiments comparing to the MMD test with a polynomial kernel are
deferred to the appendix.  

\paragraph{Convergence to Asymptotic Null.} In Figure \ref{fig:null}, we plot 
histograms of finite-sample higher-order KS test statistics and their
asymptotic null distributions, when $k=1,2$.  We considered both $P=N(0,1)$ and
\smash{$P=\mathrm{Unif}(-\sqrt{3},\sqrt{3})$} (the uniform distribution
standardized to have mean 0 and variance 1).  For a total of 1000 repetitions,
we drew two sets of samples from $P$, each of size $m=n=2000$, then computed the
test statistics. For a total of 1000 times, we also approximated the supremum of
the Gaussian process from Theorem \ref{thm:null_hks} via discretization.  We 
see that the finite-sample statistics adhere closely to their asymptotic
distributions. Interestingly, we also see that the distributions look roughly
similar across all four cases considered. Future work will examine more
thoroughly.

\begin{figure}[htb]
\centering
\includegraphics[width=0.485\columnwidth]{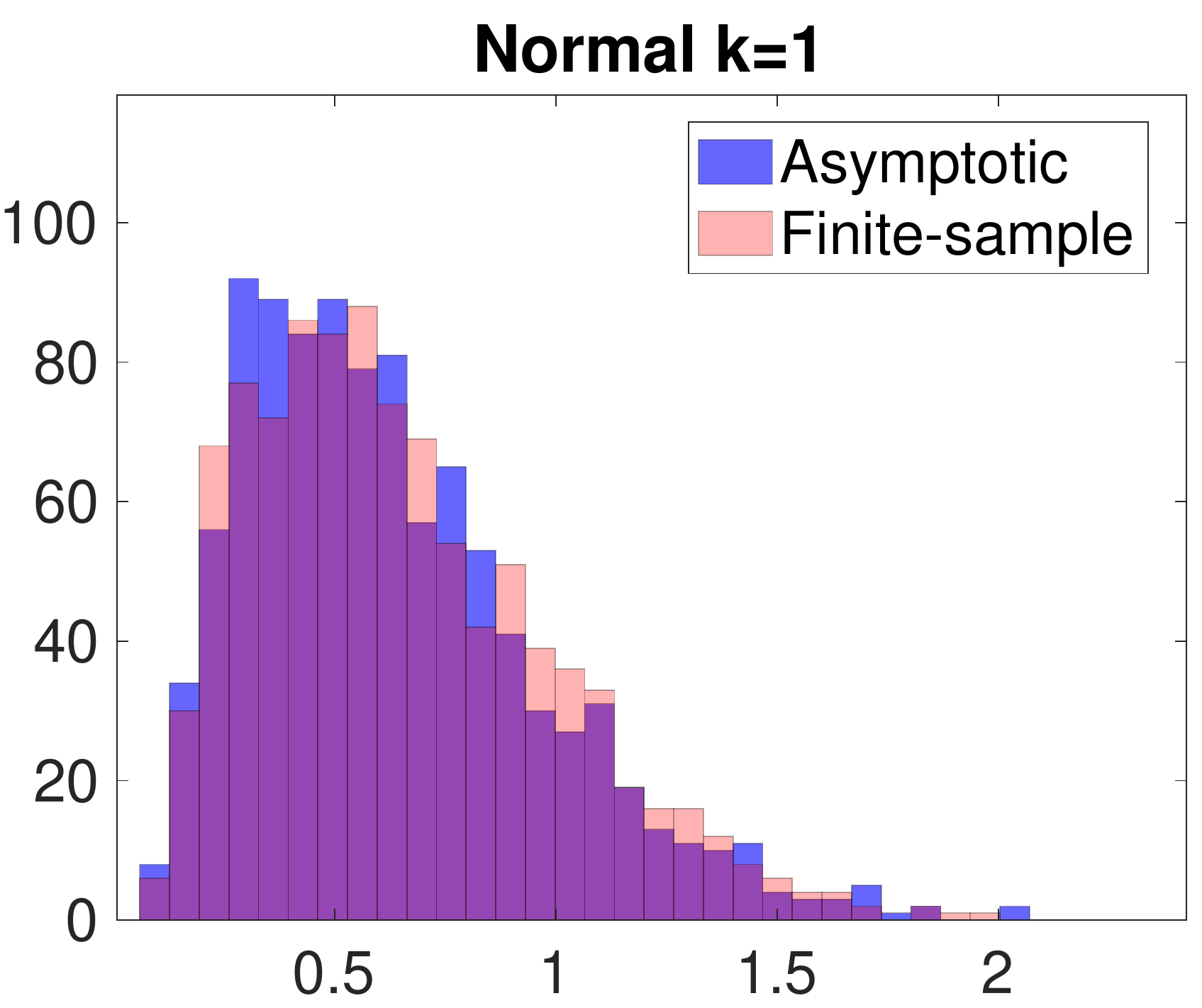}
\includegraphics[width=0.485\columnwidth]{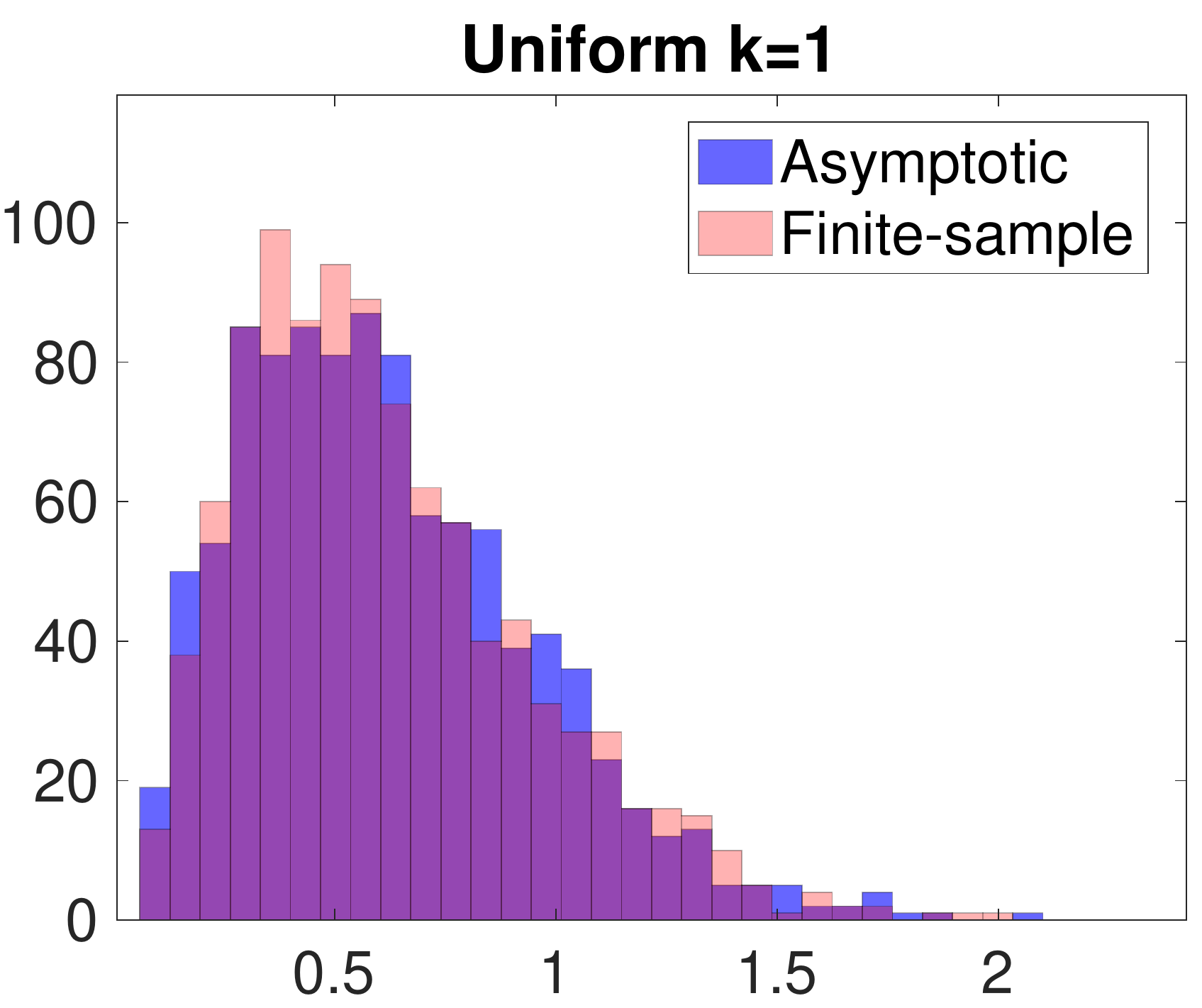}
\includegraphics[width=0.485\columnwidth]{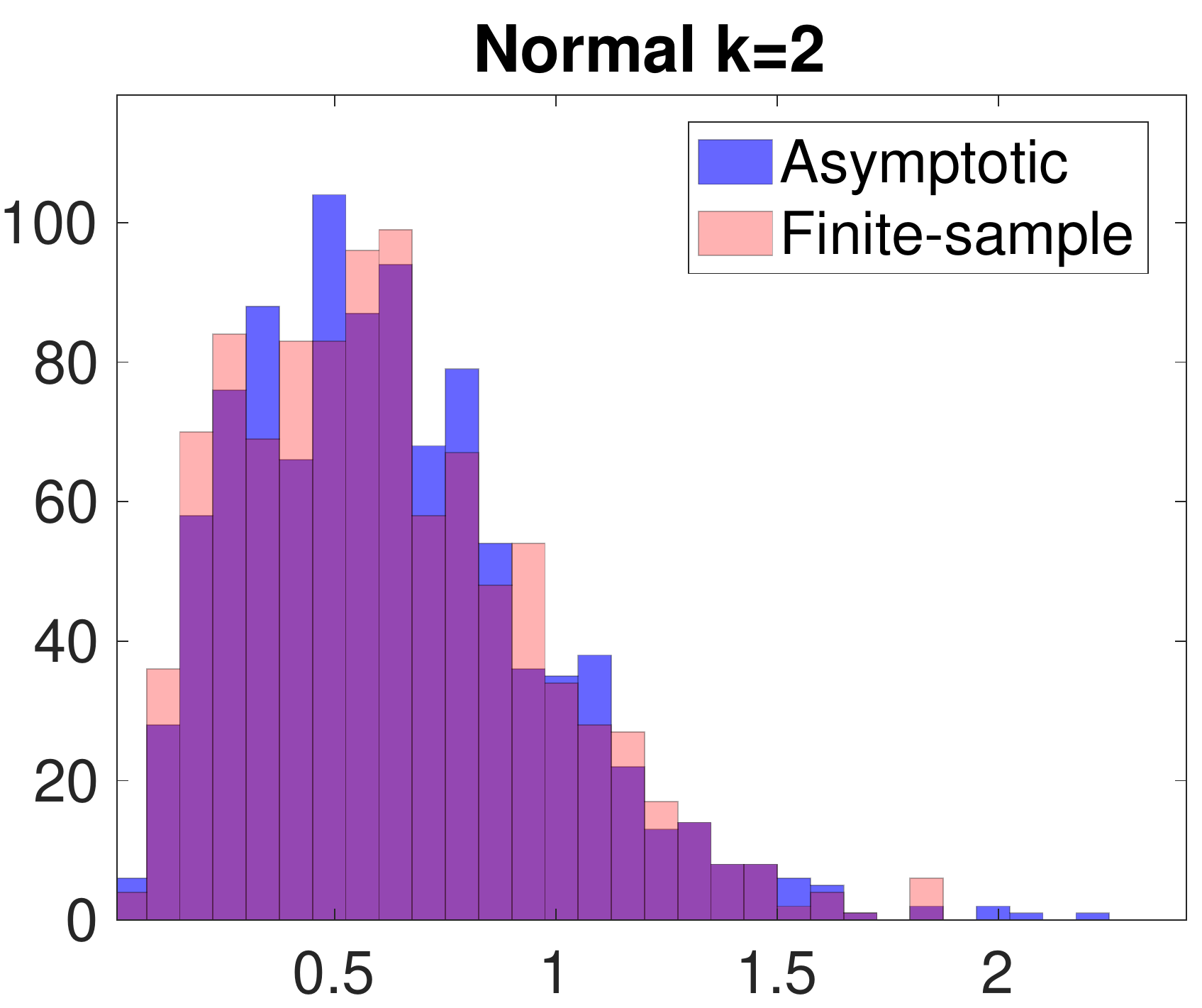}
\includegraphics[width=0.485\columnwidth]{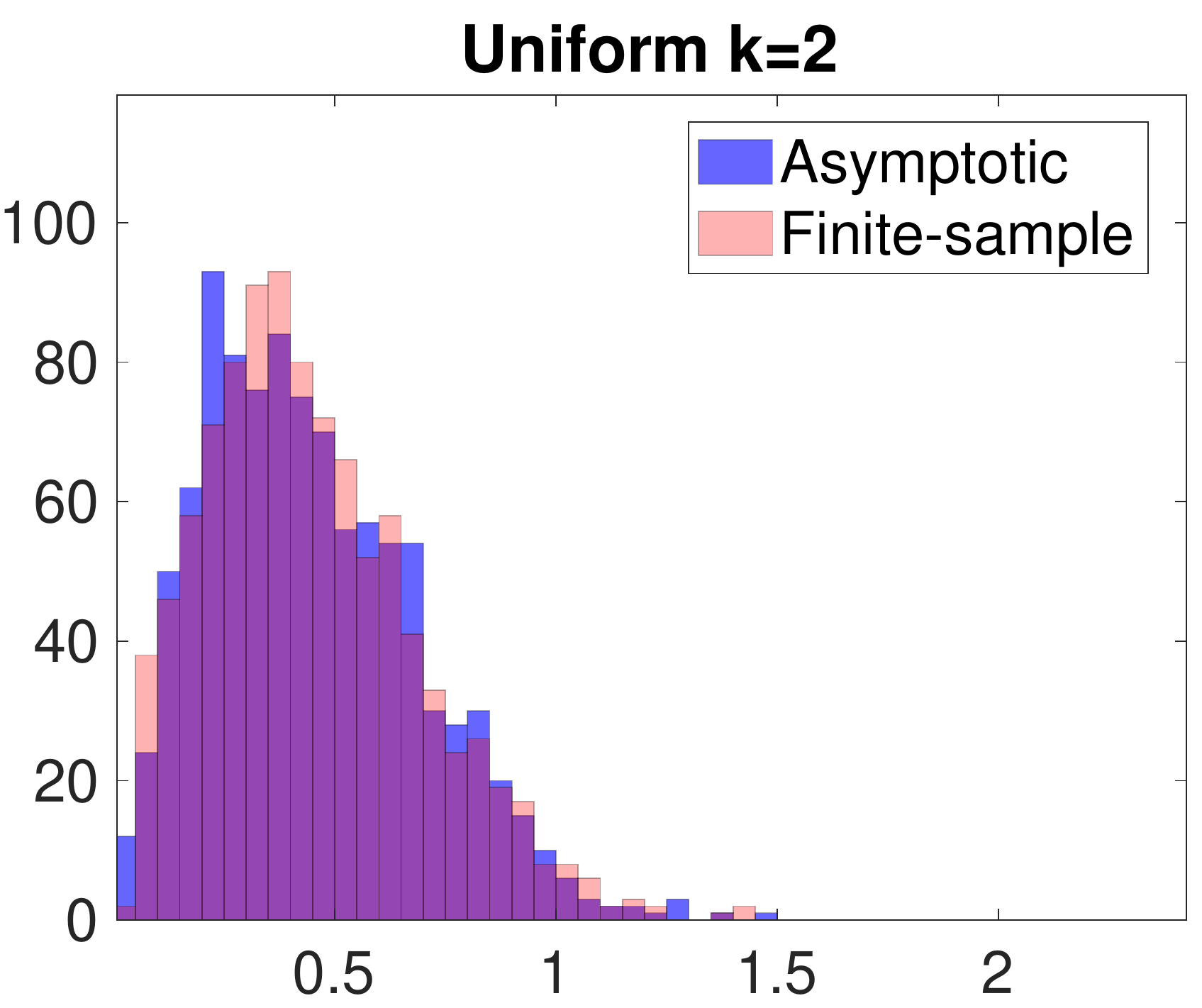}
\caption{\it\small Histograms comparing finite-sample test statistics to their
  asymptotic null distribution.} 
\label{fig:null}
\vspace{-4pt}
\end{figure}

\paragraph{Comparison to General-Purpose Tests.}

In Figures \ref{fig:gen1} and \ref{fig:gen2}, we compare the higher-order KS
tests to the KS test, and other widely-used nonparametric tests from the literature: 
the kernel maximum mean discrepancy (MMD) test \citep{gretton2012kernel} with a
Gaussian kernel, the energy distance test \citep{szekely2004testing}, and the
Anderson-Darling test \citep{anderson1954test}.  The simulation setup is
the same as that in the introduction, where we considered $P,Q$ with different 
variances, except here we study different means: $P=N(0,1)$, 
$Q=N(0.2,1)$, and different third moments: $P=N(0,1)$,
\smash{$Q=t(3)$}, where $t(3)$ denotes Student's t-distribution with 3 degrees
of freedom.  The higher-order KS tests generally perform favorably, and in each
setting there is a choice of $k$ that yields better power than KS.  In the mean
difference setting, this is $k=1$, and the power degrades for $k=3,5$, likely
because these tests are ``smoothing out'' the mean difference too much; see
Proposition \ref{prop:hks_integrated_cdfs}.  

\begin{figure}[htb]
\centering
\includegraphics[width=0.815\columnwidth]{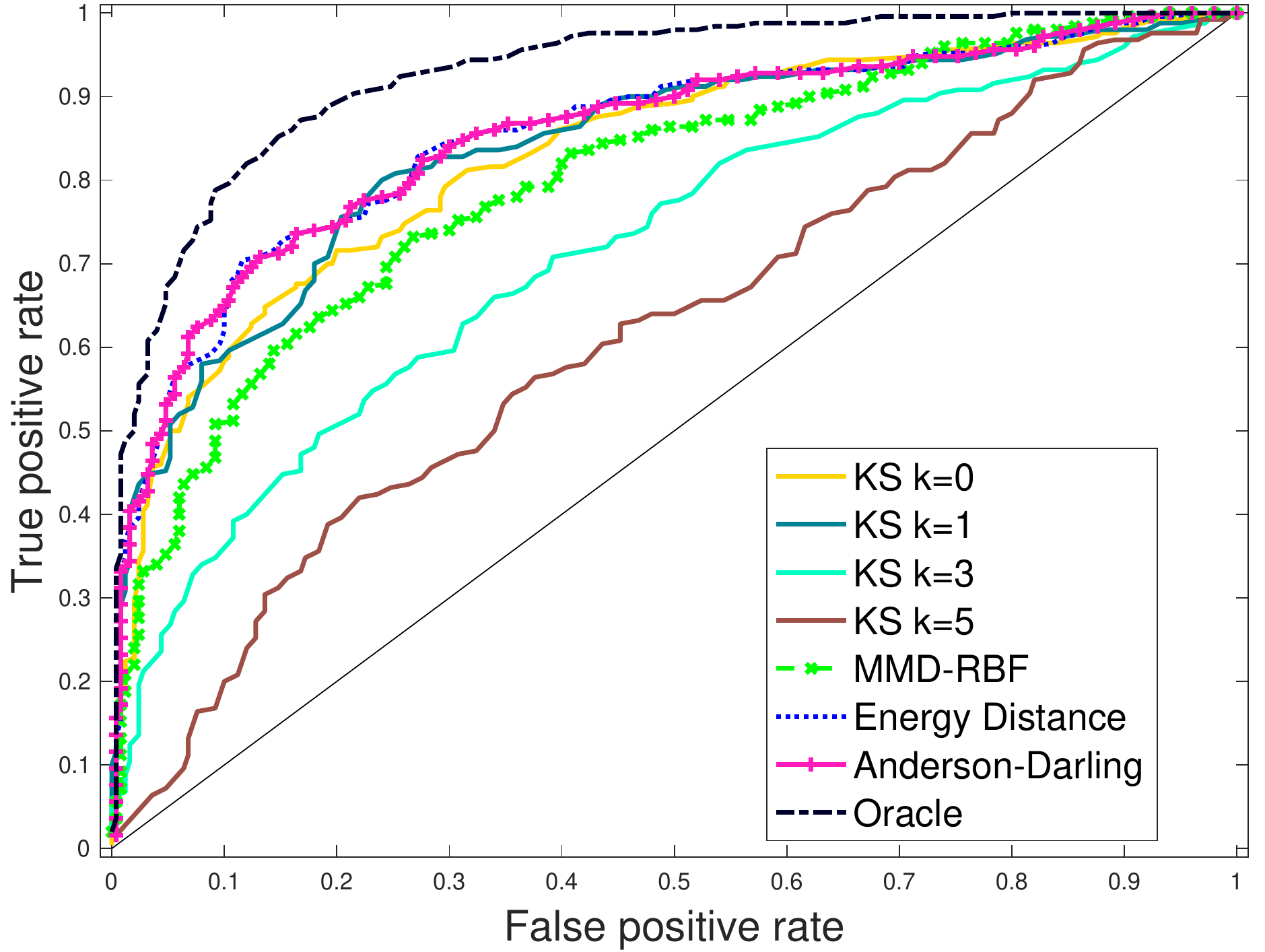}
\caption{\it\small ROC curves for $P=N(0,1)$, $Q=N(0.2,1)$.} 
\label{fig:gen1}

\smallskip\smallskip
\includegraphics[width=0.815\columnwidth]{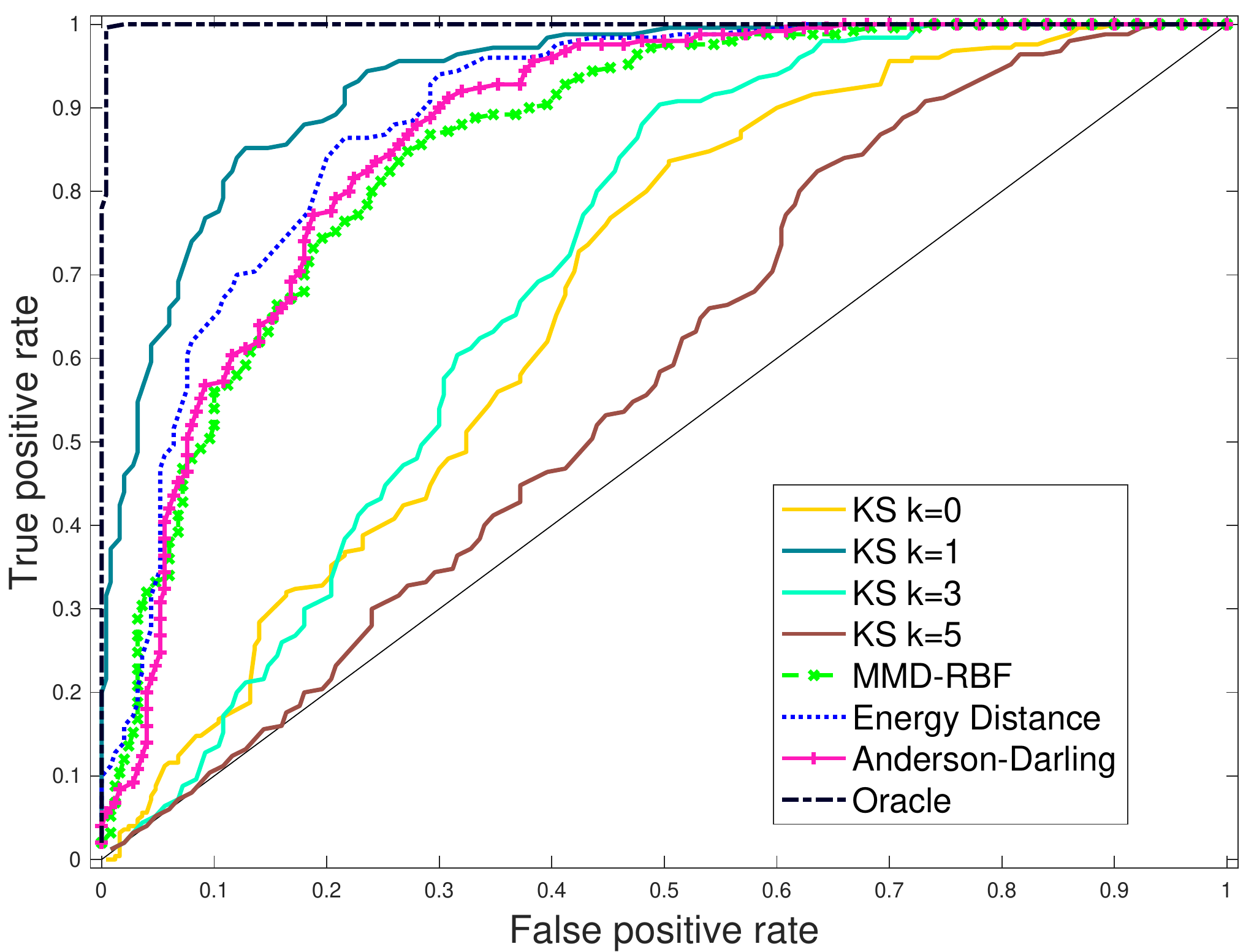}
\caption{\it\small ROC curves for $P=N(0,1)$, $Q=t(3)$.}
\label{fig:gen2} 
\vspace{-4pt}
\end{figure}

\paragraph{Local Density Differences.}  In Figures \ref{fig:loc1} and
\ref{fig:loc2}, we examine the higher-order KS tests and the KS test, in
cases where $P,Q$ have densities $p,q$ such that $p-q$ has sharp local
changes.  Figure \ref{fig:loc1} shows a case where $p-q$ 
is piecewise constant with a few short departures from 0 (see the
appendix for a plot)  and $m=n=500$. The KS test is very powerful, and the 
higher-order KS tests all perform poorly; in fact, the KS test here has better
power than all commonly-used nonparametric tests we tried (results not shown).      
Figure \ref{fig:loc2} displays a case where $p-q$ changes sharply in the  
right tail (see the appendix for a plot) and $m=n=2000$.   The power of the
higher-order KS test appears to increase with $k$, likely because the witness
functions are able to better concentrate on sharp departures for large $k$. 

\begin{figure}[htb]
\centering
\includegraphics[width=0.815\columnwidth]{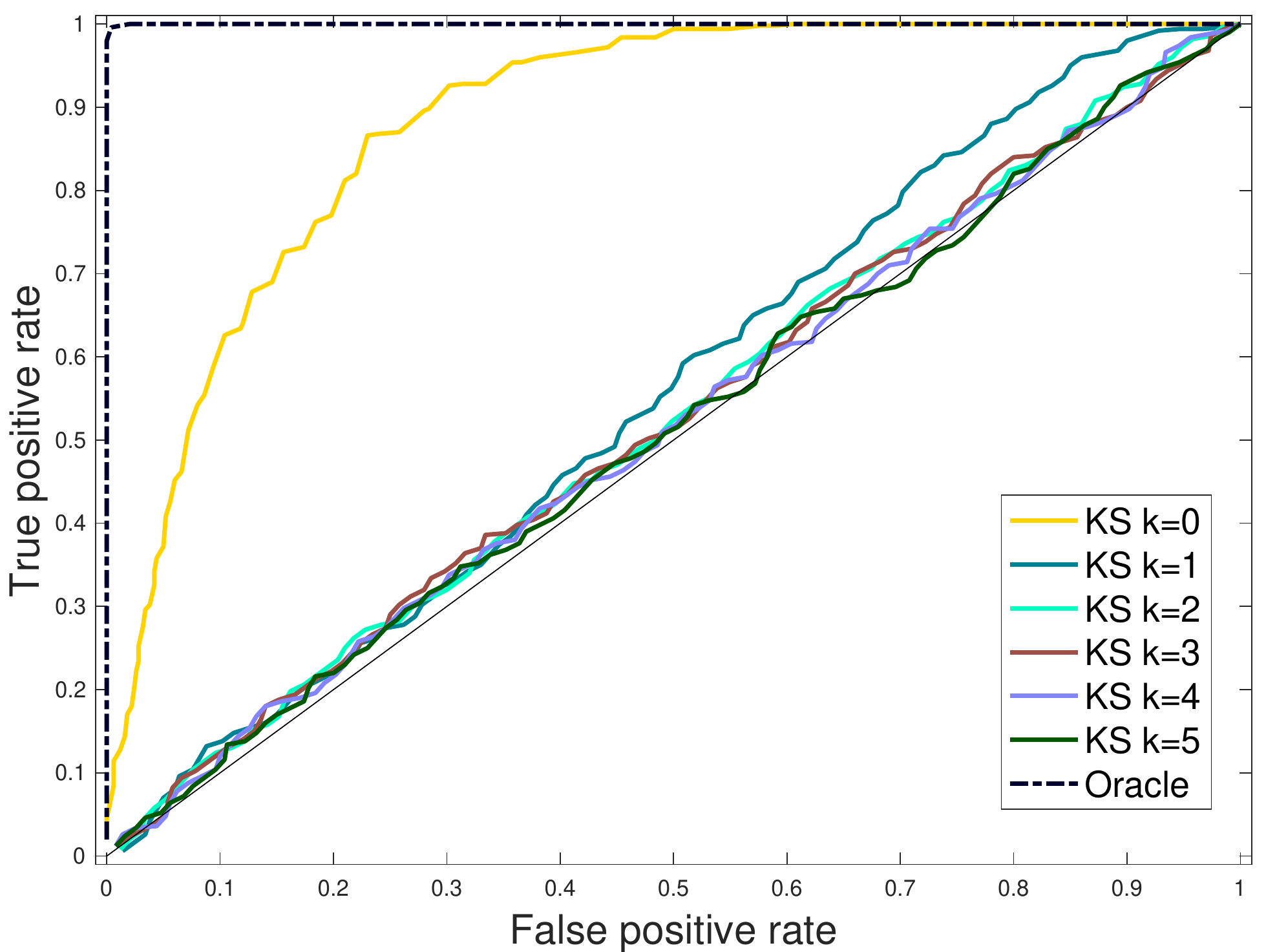}
\caption{\it\small ROC curves for piecewise constant $p-q$.} 
\label{fig:loc1}

\smallskip\smallskip
\includegraphics[width=0.815\columnwidth]{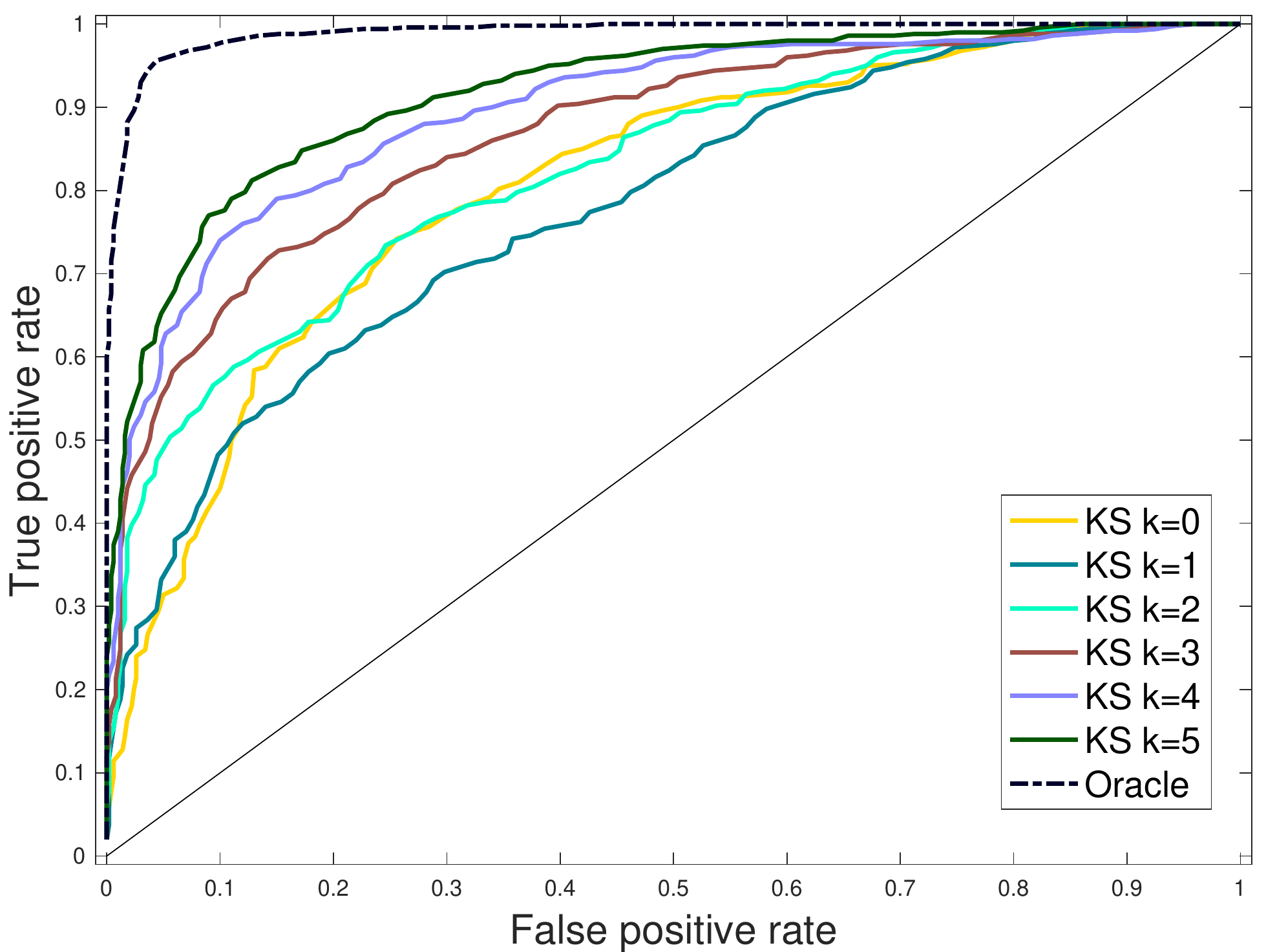} 
\caption{\it\small ROC curves for tail departure in $p-q$.} 
\label{fig:loc2}
\vspace{-4pt}
\end{figure}

\section{DISCUSSION}
\label{sec:discussion}

This paper began by noting the variational characterization of the classical KS
test as an IPM with respect to functions of bounded total variation, and then
proposed a generalization to higher-order total variation classes. This
generalization was nontrivial, with subtleties arising in defining the right
class of functions so that the statistic was finite and amenable for
simplification via a representer result, challenges in computing the statistic
efficiently, and challenges in studying asymptotic convergence and concentration
due to the fact that the function class is not uniformly sup norm bounded.
The resulting class of linear-time higher-order KS tests was shown empirically
to be more sensitive to tail differences than the usual KS test, and to have 
competitive power relative to several other popular tests. 

In future work, we intend to more formally study the power properties of
our new higher-order tests relative to the KS test. The following is a lead in
that direction.  For $k \geq 1$, define $I^k$ to be the $k$th order integral 
operator, acting on a function $f$, via
$$
(I^k f)(x) = \int_0^x \int_0^{t_k} \cdots \int_0^{t_2} f(t_1) \, dt_1 \cdots  
dt_k.
$$
Denote by $F_P,F_Q$ the CDFs of the distributions $P,Q$.  Notice that the
population-level KS test statistic can be written as $\ipm(P,Q;\cF_0) = \|F_P -
F_Q\|_\infty$, where $\|\cdot\|_\infty$ is the sup norm. Interestingly, a
similar representation holds for the higher-order KS tests.

\begin{proposition}
\label{prop:hks_integrated_cdfs}
Assuming $P,Q$ have $k$ moments,
$$
\ipm(P,Q;\cF_k) = \|(I^k)^* (F_P - F_Q) \|_\infty,
$$
where $(I^k)^*$ is the adjoint of the bounded linear operator $I^k$, with
respect to the usual $L_2$ inner product. Further, if $P,Q$ are supported on 
$[0,\infty)$, or their first $k$ moments match, then we have the more explicit 
representation 
\begin{multline*}
\ipm(P,Q;\cF_k) = \\
\sup_{x \in \R} \bigg| \int_x^\infty \int_{t_k}^\infty \cdots \int_{t_2}^\infty
(F_P-F_Q)(t_1) \, dt_1 \cdots dt_k \bigg|.
\end{multline*}
\end{proposition}

The representation in Proposition \ref{prop:hks_integrated_cdfs} could provide
one avenue for power analysis.  When $P,Q$ are supported on $[0,\infty)$, or
have $k$ matching moments, the representation is particularly simple in
form. This form confirms the intuition that detecting higher-order moment
differences is hard: as $k$ increases, the $k$-times integrated CDF difference 
$F_P - F_Q$ becomes smoother, and hence the differences are less accentuated.    

In future work, we also intend to further examine the asymptotic null of the
higher-order KS test (the Gaussian process from Theorem \ref{thm:null_hks}), and
determine to what extent it depends on the underlying distribution $P$ (beyond
say, its first $k$ moments). Lastly, some ideas in this paper seem
extendable to the multivariate and graph settings, another direction for future
work.   

\paragraph{Acknowledgments.} We thank Alex Smola for several early
inspiring discussions. VS and RT were supported by NSF Grant DMS-1554123.       

\newpage
\onecolumn
\appendix
\section{Appendix}

\subsection{Comparing the Test in \citet{wang2014falling}}

The test statistic in \citet{wang2014falling} can be expressed as
\begin{equation}
\label{eq:hks_wang2014}
T^{**} = \max_{t \in Z_{(N)}} \; |(\P_m - \Q_n) g^+_t| = 
\max_{t \in Z_{(N)}} \; \bigg| \frac{1}{m} \sum_{i=1}^m (x_i - t)_+^k - 
\frac{1}{n} \sum_{i=1}^n (y_i - t)_+^k \bigg|.
\end{equation}
This is very close to our approximate statistic $T^*$ in
\eqref{eq:hks_approx_simple}.  The only difference is that we replace   
\smash{$g^+_t(x)=(x-t)_+^k$} by \smash{$g^-_t(x) = (t-x)_+^k$} for $t \leq 0$.  

Our exact (not approximate) statistic is in \eqref{eq:hks_representer}.  This has
the advantage having an equivalent variational form \eqref{eq:hks}, and the
latter form is important because it shows the statistic to be a metric.

\subsection{Proof of Proposition \ref{prop:metric}}

We first claim that $F(x) = |x|^k/k!$ is an
envelope function for $\cF_k$, meaning $f \leq F$ for all $f \in \cF_k$.  To see
this, note each $f \in \cF_k$ has $k$th weak derivative with left or right limit 
of 0 at 0, so $|f^{(k)}(x)| \leq \TV(f^{(k)}) \leq 1$; repeatedly integrating
and applying the derivative constraints yields the claim.   Now due to the
envelope function, if $P,Q$ have $k$ moments, then the IPM is well-defined:
$|\P f| < \infty$, $|\Q f| < \infty$ for all $f \in \cF_k$.  Thus if $P=Q$, then 
clearly $\ipm(P,Q;\cF_k)=0$.

For the other direction, suppose that $\ipm(P,Q;\cF_k)=0$.  By simple rescaling,
for any $f$, if \smash{$\TV(f^{(k)}) = R > 0$}, then \smash{$\TV((f/R)^{(k)})
  \leq 1$}. Therefore $\ipm(P,Q;\cF_k)=0$ implies
\smash{$\ipm(P,Q;\widetilde\cF_k)=0$}, where 
$$
\widetilde\cF_k = \big\{ f : \TV(f^{(k)}) < \infty, \; f^{(j)}(0) = 0, \; j \in
\{0\} \cup [k-1], \; \text{and} \; f^{(k)}(0+)=0 \; \text{or} \; f^{(k)}(0-)=0
\big\}.  
$$
This also implies \smash{$\ipm(P,Q;\widetilde\cF^+_k)=0$}, where 
$$
\widetilde\cF^+_k = \{ f : \TV(f^{(k)}) < \infty, \; f(x) = 0 \; \text{for $x
  \leq 0$} \}.
$$
As the class \smash{$\widetilde\cF^+_k$} contains \smash{$C^\infty_c(\R_+)$},
where $\R_+=\{x : x > 0\}$ (and \smash{$C^\infty_c(\R_+)$} is the class of 
infinitely differentiable, compactly supported functions on $\R_+$), we have
by Lemma \ref{lem:dejan} that $P(A \cap \R_+)=Q(A \cap \R_+)$ for all open sets 
$A$.  By similar arguments, we also get that $P(A \cap \R_-)=Q(A \cap \R_-)$,
for all open sets $A$, where $\R_-=\{x : x < 0\}$. This implies that
$P(\{0\}) = Q(\{0\})$ (as $1-P(\R_+)-P(\R_-)$, and the same for $Q$), and
finally, $P(A) = Q(A)$ for all open sets $A$, which means that $P=Q$.

\subsection{Statement and Proof of Lemma \ref{lem:dejan}}

\begin{lemma}
\label{lem:dejan}
For any two distributions $P,Q$ supported on an open set $\Omega$, if 
\smash{$\E_{X \sim P}[f(X)] = \E_{Y \sim Q}[f(Y)]$} for all $f\in
C^\infty_c(\Omega)$, then $P=Q$.  
\end{lemma}

\begin{proof}
It suffices to show that $P(A) = Q(A)$ for every open set $A\subseteq \Omega$. 
As $P,Q$ are probability measures and hence Radon measures, there exists a
sequence of compact sets $K_n \subseteq A$, $n=1,2,3,\ldots$ such that 
$\lim_{n \to \infty} P(K_n) = P(A)$ and $\lim_{n \to \infty} Q(K_n) = Q(A)$.
Let $f_n$, $n=1,2,3,\ldots$ be smooth compactly supported functions with values
in $[0,1]$ such that $f_n = 1$ on $K_n$ and $f_n = 0$ outside of $A$.
(Such functions can be obtained by applying Urysohn's Lemma on appropriate sets 
containing $K_n$ and $A$  and convolving the resulting continuous function with
a bump function.) 
Then $P(K_n) \leq E_P(f_n) = E_Q(f_n) \leq Q(A)$ (where the equality by the main 
assumption in the lemma). Taking $n \to \infty$ gives $P(A) \leq Q(A)$. 
By reversing the roles of $P,Q$, we also get $Q(A) \leq P(A)$. Thus $P(A) =
Q(A)$.  
\end{proof}

\subsection{Proof of Theorem \ref{thm:hks_representer}}

Let $\cG_k$ be as in \eqref{eq:gk}.  Noting that $G_k \subseteq \cF_k$, it is
sufficient to show  
$$
\sup_{f \in \cF_k} \; |\P_m f - \Q_n f| \leq \sup_{g \in \cG_k} \; |\P_m g - \Q_n g|.
$$
Fix any $f \in \cF_k$.  Denote \smash{$Z^0_{(N)}= \{0\} \cup Z_{(N)}$}.
From the statement and proof of Theorem 1 in \citet {mammen1991nonparametric},
there exists a spline $\tf$ of degree $k$, with finite number of knots such that
for all \smash{$z \in Z^0_{(N)}$} 
\begin{align*}
f(z) &= \tf(z), \\
f^{(j)}(z) &= \tf^{(j)}(z), \;\; j \in [k-1], \\
f^{(k)}(z^+) &= \tf^{(k)}(z^+), \\
f^{(k)}(z^-) &= \tf^{(k)}(z^-).
\end{align*}
and importantly, $\TV(\tf^{(k)}) \leq \TV(f^{(k)})$.  As \smash{$0 \in
  Z^0_{(N)}$}, we hence know that the boundary constraints (derivative
conditions at 0) are met, and \smash{$\tf \in \cF_k$}.

Because \smash{$\tf$} is a spline with a given finite number of knot points, we
know that it has an expansion in terms of truncated power functions.  Write
$t_0,t_1,\ldots,t_L$ for the knots of \smash{$\tf$}, where $t_0=0$.  Also
denote $g_t = g^+_t$ when $t > 0$, and $g_t = g^-_t$ when $t < 0$.  Then for
some $\alpha_\ell \in \R$, $\ell \in \{0\} \cup [L]$, and a polynomial $p$ of
degree $k$, we have 
$$
\tf = p + \alpha_0 g_0^+ + \sum_{\ell=1}^L \alpha_\ell g_{t_\ell},
$$
The boundary conditions on \smash{$\tf$}, \smash{$g^+_0$}, \smash{$g_{t_\ell}$},
$\ell \in [L]$ imply   
\begin{gather*}
p(0) = p^{(1)}(0) = \ldots = p^{(k-1)}(0) = 0,\\
(\alpha_0 g_{0^+} + p )^{(k)}(0^+) = 0 \;\; \text{or} \;\;
(\alpha_0 g_{0^+} + p )^{(k)}(0^-) = 0.
\end{gather*}
The second line above implies that 
$$
\alpha_0 + p^{(k)} = 0 \;\; \text{or} \;\; p^{(k)} = 0.
$$
In the second case, we have $p=0$. In the first case, we have 
$p(x) = -\alpha_0 x^k / k!$, so \smash{$\alpha_0 g_0 + p = -(-1)^{k+1}  
\alpha_0 g^-_0$}.  Therefore, in all cases we can write
$$
\tf = \sum_{\ell=0}^L \alpha_\ell g_{t_\ell},
$$
with the new understanding that $g_0$ is either \smash{$g^+_0$} or
\smash{$g^-_0$}.  This means that \smash{$\tf$} lies in the span of functions in 
$\cG_k$.  Furthermore, our last expression for \smash{$\tf$} implies 
$$
\|\alpha\|_1 = \sum_{\ell=0}^L |\alpha_\ell| = \TV(\tf^{(k)}) \leq \TV(f^{(k)})
\leq 1. 
$$
Finally, using the fact that $f$ and \smash{$\tf$} agree on \smash{$Z^0_{(N)}$}, 
\begin{align*}
|\P_m f - \Q_n f| &= |\P_m \tf - \Q_n \tf| \\
&= \bigg|\sum_{\ell=0}^L \alpha_\ell (\P_m g_{t_\ell} - \Q_n g_{t_\ell}) \bigg|
  \\ 
&\leq \sum_{\ell=0}^L |\alpha_\ell| \cdot \sup_{g \in \cG_k} |\P_m g - \Q_n g|
  \\  
&\leq \sup_{g \in \cG_k} \; |\P_m g - \Q_n g|,
\end{align*}
the last two lines following from Holder's inequality, and $\|\alpha\|_1
\leq 1$.  This completes the proof.

\subsection{Proof of Proposition \ref{prop:hks_approx_k6}} 

From \citet{shor1998nondifferentiable,nesterov2000squared}, a polynomial of
degree $2d$ is nonnegative on $\R$ if and only if it can be written as a sum of
squares (SOS) of polynomials, each of degree $d$.  Crucially, one can show that 
\smash{$p(x) = \sum_{i=0}^{2d} a_i x^i$} is SOS if and only if there is a 
positive semidefinite matrix $Q \in \R^{(d+1) \times (d+1)}$ such that 
$$
a_{i-1} = \sum_{j+k=i} Q_{jk}, \;\; i \in [2d].
$$
Finding such a matrix $Q$ can be cast as a semidefinite program (SDP) (a
feasibility program, to be precise), and therefore checking nonnegativity can be 
done by solving an SDP. 

Furthermore, calculating the maximum of a polynomial $p$ is equivalent to
calculating the smallest $\gamma$ such that $\gamma - p$ is nonnegative.  This
is therefore also an SDP.

Finally, a polynomial of degree $k$ is nonnegative an interval $[a,b]$ if and
only if it can be written as 
\begin{equation}
p(x) = \begin{cases}
s(x) + (x-a)(b-x) t(x)  \quad \text{$k$ even}\\
(x-a) s(x) + (b-x) t(x) \quad \text{$k$ odd}
\end{cases},
\end{equation}
where $s,t$ are polynomials that are both SOS.  Thus maximizing a
polynomial over an interval is again equivalent to an SDP.  For details,
including a statement that such an SDP can be solved to $\epsilon$-suboptimality
in $c_k \log(1/\epsilon)$ iterations, where $c_k>0$ is a constant that depends
on $k$, see \citet{nesterov2000squared}.

\subsection{Proof of Lemma \ref{lem:hks_approx_simple}}

Suppose $t^*$ maximizes the criterion in \eqref{eq:hks_representer}.  If
$t^*=0$, then $T^*=T$ and the result trivially holds.  Assume without a loss of
generality that $t^*>0$, as the result for $t^*<0$ will follow similarly.  

If $t^*$ is one of the sample points \smash{$Z_{(N)}$}, then $T^*=T$ and the  
result trivially holds; if $t^*$ is larger than all points in \smash{$Z_{(N)}$},
then $T^*=T=0$ and again the result trivially holds.  Hence we can assume
without a loss of generality that $t^* \in (a,b)$, where \smash{$a,b \in
  Z^0_{(N)}$}. Define 
$$
\phi(t) = \frac{1}{k!} \sum_{i=1}^N c_i (z_i-t)_+^k, \;\; t \in [a,b],
$$
where $c_i=(\one_m/m-\one_n/n)_i$, $i \in [N]$, as before. Note that
$T=\phi(t^*)$, and  
$$
|\phi'(t) | \leq \frac{1}{(k-1)!} \sum_{i=1}^N |c_i | |z_i^{k-1}|
=  \frac{1}{(k-1)!} \bigg( \frac{1}{m}\sum_{i=1}^m |x_i|^{k-1}
+\frac{1}{n}\sum_{i=1}^n |y_i|^{k-1} \bigg) := L.
$$
Therefore
$$
T-T^* \leq |f(t^*)| -  |f(a)| \leq | f(t^*) - f(a) |\leq |t^*-a| L \leq
\delta_N L, 
$$
as desired. 

\subsection{Proof of Lemma \ref{lem:gk_bracket}}

Decompose \smash{$\cG_k=\cG^+_k \cup \cG^-_k$}, where \smash{$\cG^+_k=\{ g^+_t :
  t \geq 0 \}$}, \smash{$\cG^-_k=\{ g^-_t : t \leq 0 \}$}.  We will bound the 
bracketing number of \smash{$\cG^+_k$}, and the result for \smash{$\cG^-_k$},
and hence $\cG_k$, follows similarly.

Our brackets for \smash{$\cG_k^+$} will be of the form \smash{$[g_{t_i},
  g_{t_{i+1}}]$}, $i \in \{0\} \cup [R]$, where $0=t_1 < t_2 < \cdots < t_{R+1}
= \infty$ are to be specified, with the convention that $g_\infty = 0$.  
It is clear that such a set of brackets covers \smash{$\cG_k^+$}.
Given $\epsilon > 0$, we need to choose the brackets such that 
\begin{equation}
\label{eq:eps_bd}
\| g_{t_i} - g_{t_{i+1}} \|_2 \leq \epsilon, \;\; i\in \{0\} \cup [R], 
\end{equation}
and then show that the number of brackets $R$ is small enough to satisfy the bound
in the statement of the lemma. 

For any $0 \leq s < t$, 
\begin{align*}
k!^2 \|g_s - g_t\|_2^2
&= \int_s^t (x-s)_+^{2k} \,dP(x) + 
  \int_t^\infty  \big( (x-s)^k - (x-t)^k \big)^2 \, dP(x) \\
&\leq \int_s^\infty  \big( k (x-s)^{k-1} (t-s) \big)^2 \, dP(x) \\ 
&= k^2 (t-s)^2 \int_s^\infty  (x-s)^{2k-2} \, dP(x),
\end{align*}
where the second line follows from elementary algebra. 
Now in view of the moment bound assumption, we can bound the integral above
using Holder's inequality with $p=(2k+\delta)/(2k-2)$ and
$q=(2k+\delta)/(2+\delta)$ to get  
\begin{align}
\nonumber
k!^2 \|g_s - g_t \|_2^2
&\leq k^2 (t-s)^2 \bigg( \int_s^\infty (x-s)^{2k+\delta} \,dP(x) \bigg)^{1/p}
  \bigg( \int_s^\infty 1^q \,(x) dP\bigg)^{1/q} \\
\label{eq:gs_gt}
&\leq \frac{M^{1/p}}{(k-1)!^2} (t-s)^2,
\end{align}
where recall the notation \smash{$M=\E[|X|^{2k+\delta}] < \infty$}. 

Also, for any $t>0$, using Holder's inequality again, we have
\begin{align}
\nonumber
k!^2 \|g_t - 0\|_2^2
&= \int_t^\infty (x-t)^{2k} \, dP(x) \\
\nonumber
&\leq \bigg(\int_t^\infty (x-t)^{2k+\delta} \, dP(x) \bigg)^{{2k/(2k+\delta)}} 
  \big(P(X\geq t) \big)^{\delta/(2k+\delta)} \\
\label{eq:gt_0}
&\leq M^{2k/(2k+\delta)} \bigg( \frac{\E |X|^{2k+\delta)}}{t^{2k+\delta}}
  \bigg)^{\delta/(2k+\delta)} 
=\frac{M}{t^\delta},
\end{align}
where in the third line we used Markov's inequality.

Fix an $\epsilon > 0$.  For parameters $\beta,R > 0$ to be determined, set $t_i
= (i-1) \beta$ for $i \in [R]$ and $t_0=0$, $t_{R+1} = \infty$.  Looking at
\eqref{eq:gs_gt}, to meet \eqref{eq:eps_bd}, we see we can choose $\beta$ such
that   
$$
\frac{M^{1/p}}{(k-1)!^2} \beta^2 \leq \epsilon^2.
$$
Then for such a $\beta$, looking at \eqref{eq:gt_0}, we see we can choose $R$
such that 
$$
\frac{M}{ k!^2 ((R-1)\beta)^\delta } \leq \epsilon^2.
$$
In other words, we can choose choose
$$
\beta = \frac{(k-1)!}{M^{1/2p}}, \;\;
R= 1 + \left\lceil \frac{M^{1/2p+1/\delta}}{(k-1)! k!^{2/\delta}
    \epsilon^{2/\delta+1}} \right\rceil, 
$$
and \eqref{eq:gs_gt}, \eqref{eq:gt_0} imply that we have met \eqref{eq:eps_bd}.  
Therefore, 
$$
 \log N_{[]}(\epsilon, \|\cdot\|, \cG_k^+) \leq \log R \leq
C \log \frac{M^{1+ \frac{\delta(k-1)}{2k+\delta}}}{\epsilon^{2+\delta}},
$$
where $C>0$ depends only on $k,\delta$.

\subsection{Proof of Theorem \ref{thm:null_hks}}

Once we have a finite bracketing integral for $\cG_k$, we can simply apply
Theorem \ref{thm:null_gen} to get the result.  Lemma \ref{lem:gk_bracket}
shows the log bracketing number of $\cG_k$ to grow at the rate
$\log(1/\epsilon)$, slow enough to imply a finite bracketing integral
(the bracketing integral will be finite as long as the log bracketing number
does not grow faster than $1/\epsilon^2$). 

\subsection{Proof of Corollaries \ref{cor:null_hks_approx_k6} and 
  \ref{cor:null_hks_approx_simple}} 

For the approximation from Proposition \ref{prop:hks_approx_k6}, observe 
$$
\sqrt{N} T_\epsilon = \sqrt{N} T + \sqrt{N} (T - T_\epsilon),
$$
and \smash{$0 \leq \sqrt{N} (T - T_\epsilon) \leq \sqrt{N} \epsilon$}, so for
\smash{$\epsilon=o(1/\sqrt{N})$}, we will have \smash{$\sqrt{N} T_\epsilon$}
converging weakly to the same Gaussian process as \smash{$\sqrt{N} T$}. 

For the approximation in \eqref{eq:hks_approx_simple}, the argument is similar,
and we are simply invoking Lemma 5 in \citet{wang2014falling} to bound the
maximum gap $\delta_N$ in probability, under the density conditions.

\subsection{Proof of Theorem \ref{thm:concentration_hks}}

Let \smash{$W=\sqrt{m} \ipm(P_m,P; \cG_k)$}.  The bracketing integral of $\cG_k$
is finite due to the slow growth of the log bracketing number from Lemma
\ref{lem:gk_bracket}, at the rate $\log(1/\epsilon)$.  Also, we can clearly take 
$F(x)=|x|^k/k!$ as an envelope function for $\cG_k$.  Thus, we can apply Theorem 
\ref{thm:concentration_hks} to yield 
$$
\big(\E[\ipm(P_m,P; \cG_k)^p]\big)^{1/p} \leq \frac{C}{\sqrt{m}} 
$$
for a constant $C>0$ depending only on $k,p$, and $\E|X|^p$.  Combining this
with Markov's inequality, for any $a$,
$$
\P\big(\ipm(P_m,P;\cG_k) > a \big) \leq \bigg(\frac{C}{\sqrt{m} a}\bigg)^p,
$$
thus for \smash{$a=C/(\sqrt{m}\alpha^{1/p})$}, we have $\ipm(P_m,P; \cG_k) \leq
a$ with probability at least $1-\alpha$.  The same argument applies to 
\smash{$W=\sqrt{n} \ipm(Q_n,P; \cG_k)$}, and putting these together yields the
result. The result when we additionally assume finite Orlicz norms is also
similar.  

\subsection{Proof of Corollary \ref{cor:hks_representer_pop}}

Let $f$ maximize $|(\P-\Q)f|$.  Due to the moment conditions (see the proof of
Proposition \ref{prop:metric}), we have $|\P f| < \infty$, $|\Q f| < \infty$.
Assume without loss of generality that $(\P-\Q)f > 0$.  By the strong law of
large numbers, we have $(\P_m-\Q_n)f \to (\P-\Q)f$ as $m,n \to \infty$, almost
surely. Also by the strong law, $\P_m |x|^{k-1} \to \P |x|^{k-1}$ as $m \to
\infty$, almost surely, and $\Q_n |y|^{k-1} \to \Q |y|^{k-1}$ as $n \to \infty$,
almost surely. For what follows, fix any samples \smash{$X_{(m)},Y_{(n)}$}
(i.e., take them to be nonrandom) such that the aforementioned convergences
hold.   

For each $m,n$, we know by the representer result in Theorem
\ref{thm:hks_representer} that there exists $g_{mn} \in \cG_k$ such that 
$(\P_m-\Q_n)f = |(\P_m-\Q_n)g_{mn}|$.  (This is possible since the proof of
Theorem \ref{thm:hks_representer} does not rely on any randomness that is
inherent to \smash{$X_{(m)},Y_{(n)}$}, and indeed it holds for any fixed sets of
samples.)   Assume again without a loss of generality that
$(\P_m-\Q_n)g_{mn}>0$. Denote by $t_{mn}$ the knot of $g_{mn}$ (i.e.,
\smash{$g_{mn} = g^+_{t_{mn}}$} if $t \geq 0$, and \smash{$g_{mn} =
  g^-_{t_{mn}}$} if $t \leq 0$).  We now consider two cases.    

If $|t_{mn}|$ is a bounded sequence, then by the Bolzano-Weierstrass theorem, it
has a convergent subsequence, which converges say to $t \geq 0$. Passing to this
subsequence (but keeping the notation unchanged, to avoid unnecessary clutter)
we claim that $(\P_m-\Q_n) g_{mn} \to (\P-\Q) g$ as $m,n \to \infty$, where
\smash{$g=g^+_t$}. To see this, assume $t_{mn} \geq t$ without a loss of
generality (the arguments for $t_{mn} \leq t$ are similar), and note 
$$
g(x) - g_{mn}(x) = 
\begin{cases}
0 & x < t \\
(x-t)^k & t \leq x < t_{mn} \\
(t_{mn}-t) \sum_{i=0}^{k-1} (x-t)^i (x-t_{mn})^{k-1-i} & x \geq t_{mn} 
\end{cases},
$$
where we have used the identity \smash{$a^k-b^k = (a-b) \sum_{i=0}^{k-1} a^i
  b^{k-1-i}$}.  Therefore, as $m,n \to \infty$,
$$
|\P_m(g_{mn} - g)| \leq k |t_{mn}-t| \P_m |x|^{k-1} \to 0,
$$ 
because $t_{mn} \to t$ by definition, and $\P_m |x|^{k-1} \to \P|x|^k$.
Similarly, as $m,n \to \infty$, we have $|\Q_n(g_{mn} - g)| \to 0$, and
therefore $|(\P_m-\Q_n) (g_{mn} - g)| \leq |\P_m (g_{mn} - g)| + |\Q_m (g_{mn} -
g)| \to 0$, which proves the claim.  But since $(\P_m-\Q_n) g_{mn} =
(\P_m-\Q_n)f$ for each $m,n$, we must have $(\P-\Q)g = (\P-\Q)f$, i.e., there is
a representer in $\cG_k$, as desired. 

If $|t_{mn}|$ is unbounded, then pass to a subsequence in which $t_{mn}$
converges say to $\infty$ (the case for convergence to $-\infty$ is similar). In
this case, we have $(\P_m-Q_n)g_{mn} \to 0$ as $m,n \to \infty$, and since
$(\P_m-\Q_n) g_{mn} = (\P_m-\Q_n)f$ for each $m,n$, we have $(\P-\Q)f=0$.  But
we can achieve this with \smash{$(\P-\Q) g^+_t$}, by taking $\to \infty$, so
again we have a representer in $\cG_k$, as desired.

\subsection{Proof of Corollary \ref{cor:powerful}}

When we reject as specified in the corollary, note that for $P=Q$, we have type
I error at most $\alpha_N$ by Theorem \ref{thm:concentration_gen}, and as
$\alpha_N = o(1)$, we have type I error converging to 0.

For $P \not= Q$, such that the moment conditions are met, we know by Corollary  
\ref{cor:hks_representer_pop} that $\ipm(P,Q;\cG_k) \not= 0$.  Recalling 
$1/\alpha_N = o(N^{p/2})$, we have as $N \to \infty$, 
$$
c(\alpha_N) \bigg(\frac{1}{\sqrt{m}} + \frac{1}{\sqrt{n}} \bigg) 
= \alpha^{-1/p} \bigg(\frac{1}{\sqrt{m}} + \frac{1}{\sqrt{n}}\bigg) \to 0. 
$$
The concentration result from Theorem \ref{thm:concentration_hks} shows that $T$
will concentrate around $\ipm(P,Q;\cG_k) \not= 0$ with probability tending to 1,
and thus we reject with probability tending to 1.

\subsection{Additional Experiments}

\subsection{Local Density Differences Continued}

Figure \ref{fig:more} plots the densities used for the local density
difference experiments, with the left panel corresponding to Figure
\ref{fig:loc1}, and the right panel to Figure \ref{fig:loc2}.

\begin{figure}[htb]
\centering
\includegraphics[width=0.485\textwidth]{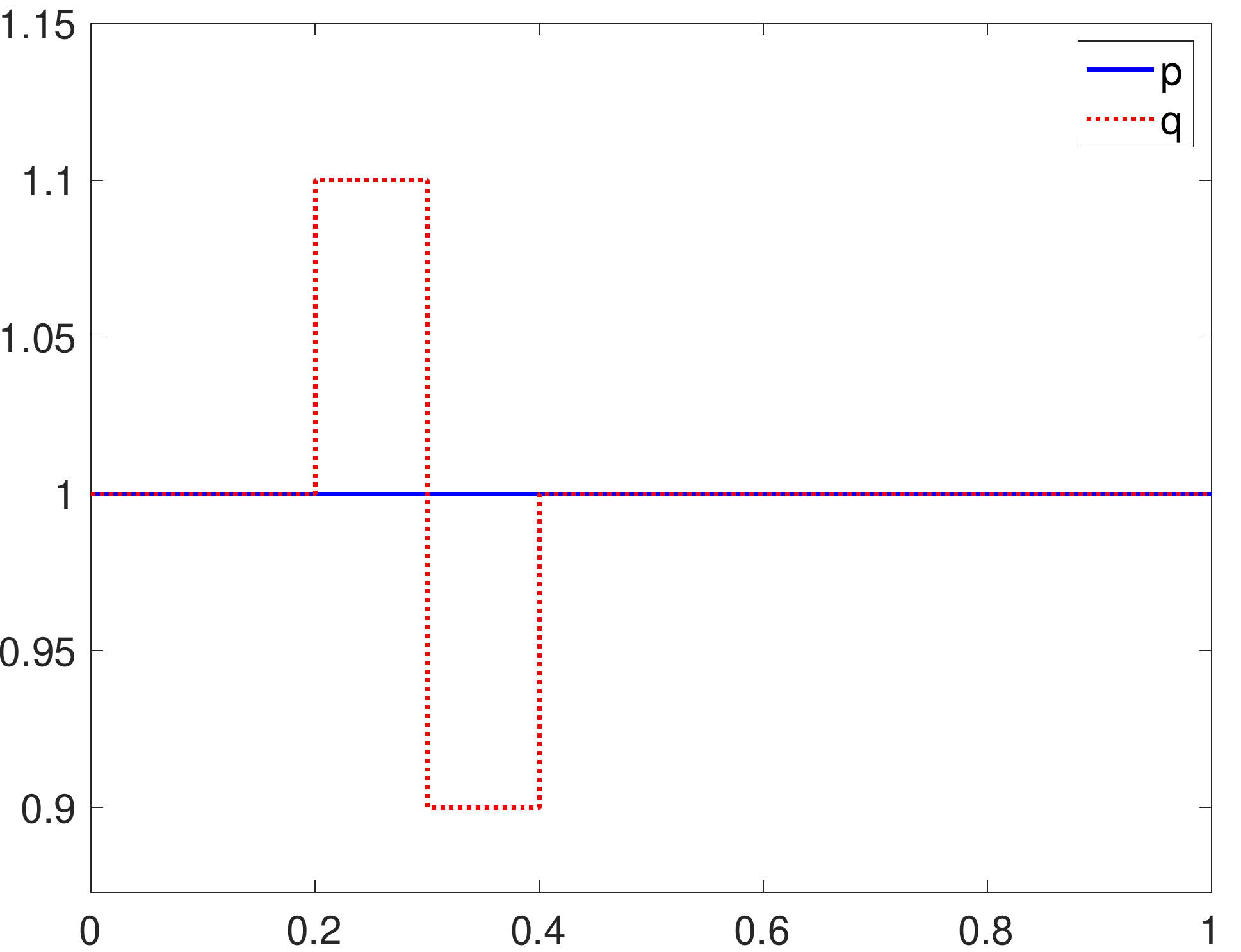}
\includegraphics[width=0.485\textwidth]{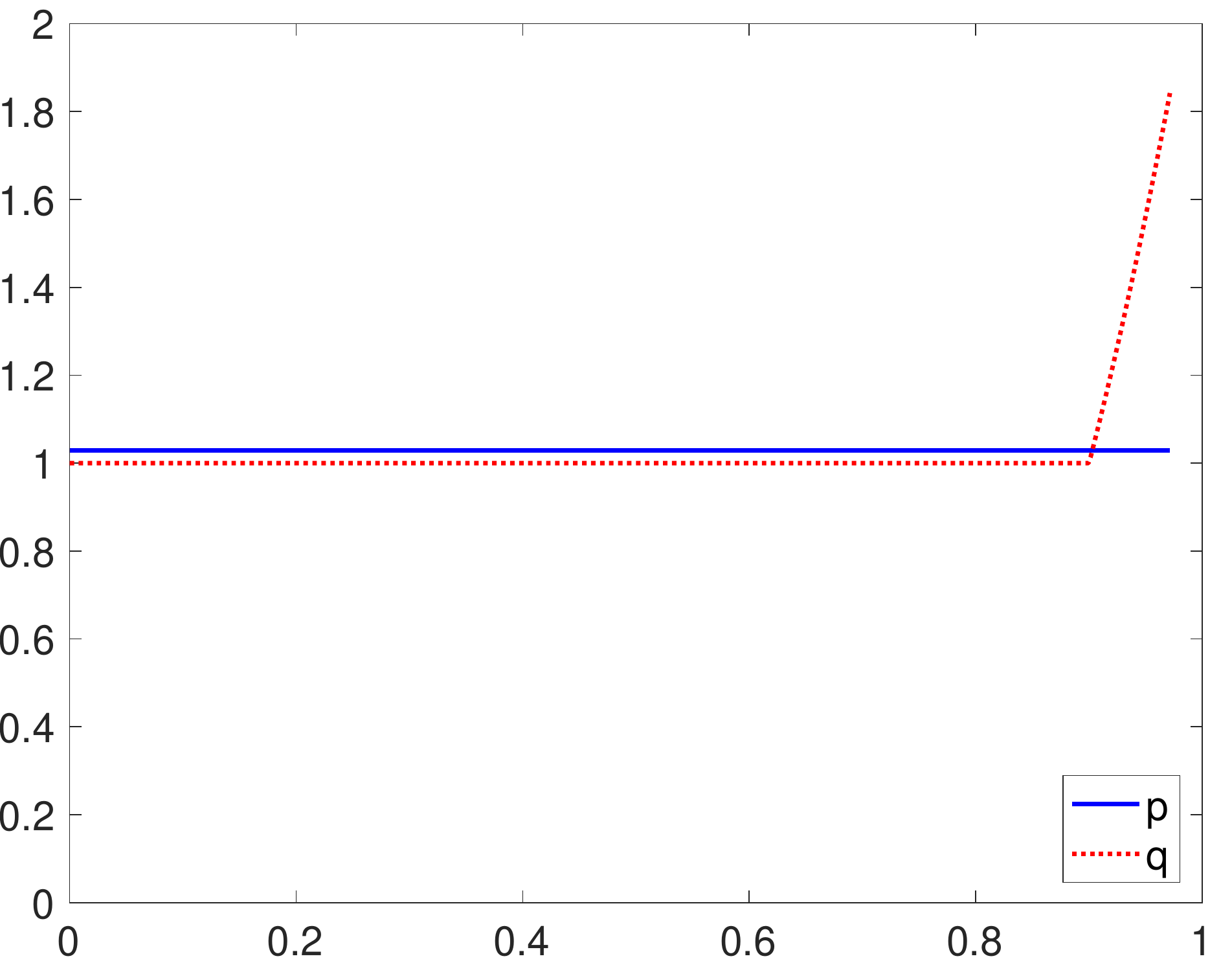}
\caption{\it\small Densities for the local density difference experiments.}
\label{fig:more}
\end{figure}

\subsection{Comparison to MMD with Polynomial Kernel}

Now we compare the higher-order KS test to the MMD test with a polynomial
kernel, as suggested by a referee of this paper.  The MMD test with a polynomial
kernel looks at moment differences up to some prespecified order $d \geq 1$, and
its test statistic can be written as  
$$
\sum_{i=0}^d {d \choose i} (\P_n x^i - \P_m y^i)^2.
$$
This looks at a weighted sum of {\it all} moments up to order $d$, whereas our
higher-order KS test looks at truncated moments of a {\it single} order $k$.  
Therefore, to put the methods on more equal footing, we aggregated the 
higher-order KS test statistics up to order $k$, i.e., writing $T_i$ to denote
the $i$th order KS test statistic, $i \in [k]$, we considered
$$
\sum_{i=0}^k {k \choose i} T_i^2,
$$
borrowing the choice of weights from the MMD polynomial kernel test statistic. 

Figure \ref{fig:poly} shows ROC curves from two experiments comparing the
higher-order KS test and MMD polynomial kernel tests.  We used distributions
$P=N(0,1)$, $Q=N(0.2,1)$ in the left panel (as in Figure \ref{fig:gen1}), and
$P=N(0,1)$, $Q=t(3)$ in the right panel (as in Figure \ref{fig:gen2}). We can see
that the (aggregated)  higher-order KS tests and MMD polynomial kernel tests
perform roughly similarly.   

\begin{figure}[htb]
\centering
\includegraphics[width=0.485\textwidth]{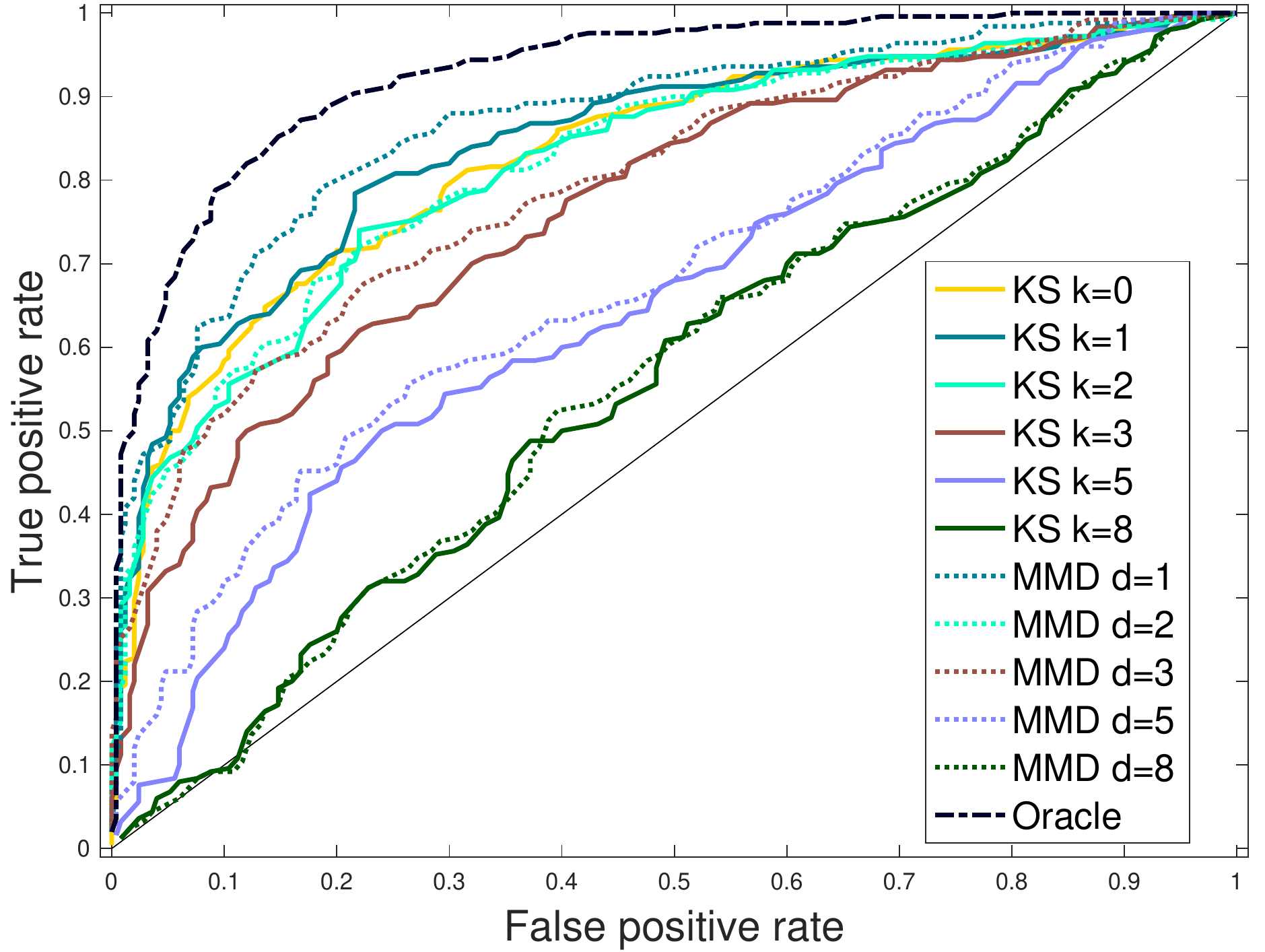}
\includegraphics[width=0.485\textwidth]{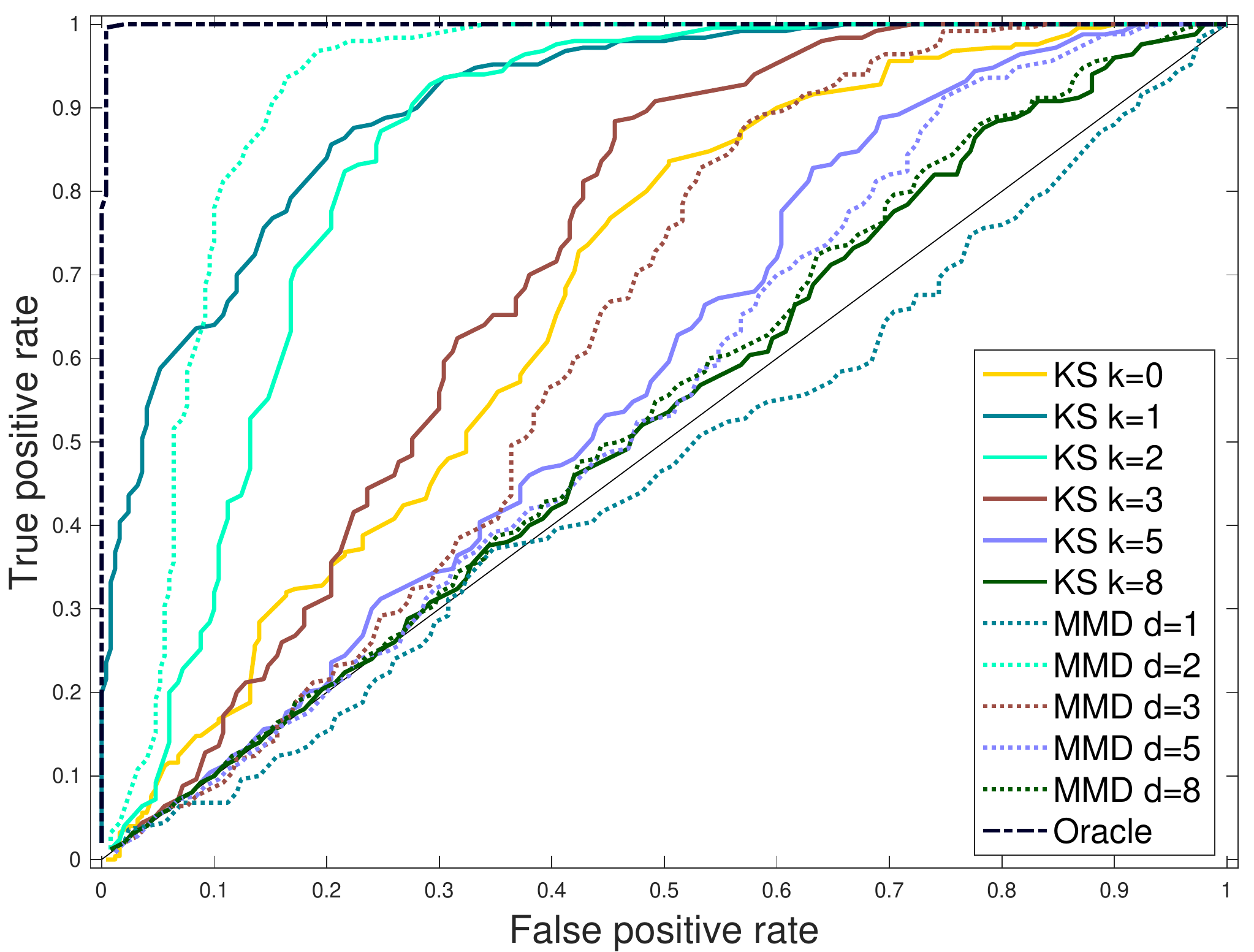}
\caption{\it\small ROC curves for $P=N(0,1)$, $Q=N(0.2,1)$ (left), and
  $P=N(0,1)$, $Q=t(3)$ (right).} 
\label{fig:poly}
\end{figure}

There is one important point to make clear: the population MMD test with a  
polynomial kernel is {\it not} a metric, i.e., there are distributions $P \not=
Q$ for which the population-level test statistic is exactly 0.  This is because  
it only considers moment differences up to order $d$, thus any pair of
distributions $P,Q$ that match in the first $d$ moments but differ in (say) 
the $(d+1)$st will lead to a population-level statistic that 0.  In this sense,
the MMD test with a polynomial kernel is not truly nonparametric, whereas the KS
test, the higher-order KS tests the MMD test with a Gaussian kernel, the energy
distance test, the Anderson-Darling test, etc., all are. 

\subsection{Proof of Proposition \ref{prop:hks_integrated_cdfs}}

For $k \geq 1$, recall our definition of $I^k$ the $k$th order integral
operator,   
$$
(I^k f)(x) = \int_0^x \int_0^{t_k} \cdots \int_0^{t_2} f(t_1) \, dt_1 \cdots 
dt_k, 
$$
Further, for $k \geq 1$, denote by $D^k$ the $k$th order derivative
operator,  
$$
(D^k f)(x) = f^{(k)}(x), 
$$
Is it not hard to check that over all functions $f$ with $k$ weak derivatives,
and that obey the boundary conditions
$f(0)=f'(0)=\cdots=f^{(k-1)}(0)=0$, these two operators act as
inverses, in that  
$$
D^k I^k f = f, \;\; \text{and} \;\; I^k D^k f = f.
$$

For a measure $\mu$, denote $\langle f, d\mu\rangle = \int f(x) \,
d\mu(x)$. (This is somewhat of an abuse of the notation for the usual $L_2$
inner product on square integrable functions, but it is convenient for what
follows.)  With this notation, we can write the $k$th order KS test statistic,
at the population-level, as  
\begin{align}
\nonumber
\sup_{f \in \cF_k} \; |\P f - \Q f| 
&= \sup_{f \in \cF_k} \; |\langle f, dP-dQ \rangle| \\
\nonumber
&= \sup_{f \in \cF_k} \; |\langle I^k D^k f, dP-dQ \rangle| \\
\nonumber
&= \sup_{\substack{h : \TV(h) \leq 1, \\ h(0+)=0 \, \text{or} \, h(0-) = 0}} \;
  |\langle I^k h, dP-dQ \rangle| \\ 
\nonumber
&= \sup_{\substack{h : \TV(h) \leq 1, \\ h(0+)=0 \, \text{or} \, h(0-) = 0}} \;
  |\langle h, (I^k)^* (dP-dQ) \rangle| \\ 
\label{eq:hks_adjoint}
&= \| (I^1)^* (I^k)^* (dP-dQ) \|_\infty. 
\end{align}
In the second line, we used the fact that $I^k$ and $D^k$ act as inverses over
$f \in \cF_k$ because these functions all satisfy the appropriate boundary 
conditions.  In the third line, we simply reparametrized via $h=f^{(k)}$.  In
the fourth line, we introduced the adjoint operator $(I^k)^*$ of $I^k$ (which
will be described in detail shortly).  In the fifth line, we leveraged the
variational result for the KS test ($k=0$ case), where $(I^1)^*$ denotes the
adjoint of the integral operator $I^1$ (details below), and we note that the
limit condition at 0 does do not affect the result here.

We will now study the adjoints corresponding to the integral operators.  By 
definition $(I^1)^* g$ must satisfy for all functions $f$ 
$$
\langle I^1 f, g \rangle = \langle f, (I^1)^* g \rangle.
$$
We can rewrite this as 
$$
\int \int_0^x f(t) g(x) \, dt \, dx = 
\int f(t) ((I^1)^* g)(t) \, dt,
$$
and we can recognize by Fubini's theorem that therefore
$$
((I^1)^* g)(t) = \begin{cases}
\displaystyle \int_t^\infty g(x) \, dx & t \geq 0 \\
\displaystyle -\int_{-\infty}^t g(x) \, dx & t < 0.
\end{cases}
$$
For functions $g$ that integrate to 0, this simplifies to  
\begin{equation}
\label{eq:int_adjoint}
((I^1)^* g)(t) = \int_t^\infty g(x) dx, \;\; t \in \R.
\end{equation}

Returning to \eqref{eq:hks_adjoint}, because we can decompose $I^k = I^1 I^1  
\cdots I^1$ ($k$ times composition), it follows that $(I^k)^* = (I^1)^* (I^1)^*
\cdots (I^1)^*$  ($k$ times composition), so  
$$
\| (I^1)^* (I^k)^* (dP-dQ) \|_\infty = \| (I^k)^* (I^1)^* (dP-dQ) \|_\infty  
= \| (I^k)^* (F_P-F_Q) \|_\infty,
$$
where in the last step we used \eqref{eq:int_adjoint}, as $dP-dQ$ integrates to
0.  This proves the first result in the proposition.

To prove the second result, we will show that 
$$
 (I^k)^* (F_P-F_Q) (x) = \int_x^\infty \int_{t_k}^\infty \cdots
 \int_{t_2}^\infty (F_P-F_Q)(t_1) \, dt_1 \cdots dt_k,
$$
when $P,Q$ has nonnegative supports, or have $k$ matching moments.  In the first
case, the above representation is clear from the definition of the adjoint.  In
the second case, we proceed by induction on $k$.  For $k=1$, note that $F_P-F_Q$
integrates to 0, which is true because    
$$
\langle 1, F_P-F_Q \rangle = \langle 1, (I^1)^*(dP-dQ) \rangle = \langle x,
dP-dQ \rangle = 0, 
$$
the last step using the fact that $P,Q$ have matching first moment.  Thus,
as $F_P-F_Q$ integrates to 0, we can use \eqref{eq:int_adjoint} to see that 
$$
(I^1)^* (F_P-F_Q) (x) = \int_x^\infty (F_P-F_Q)(t) \, dt.
$$
Assume the result holds for $k-1$. We claim that $(I^{k-1})^*(F_P-F_Q)$
integrates to 0, which is true as 
$$
\langle 1, (I^{k-1})^*(F_P-F_Q) \rangle = \langle 1, (I^k)^*(dP-dQ) \rangle = 
\langle x^k/k!, dP-dQ \rangle = 0, 
$$
the last step using the fact that $P,Q$ have matching $k$th moment.  Hence,
as $(I^{k-1})^*(F_P-F_Q)$ integrates to 0, we can use \eqref{eq:int_adjoint} and
conclude that 
\begin{align*}
(I^k)^* (F_P-F_Q) (x) &= (I^1)^*(I^{k-1})^* (F_P-F_Q) (x) \\
&= \int_x^\infty (I^{k-1})^* (F_P-F_Q)(t) \, dt \\
&= \int_x^\infty \int_{t_k}^\infty \cdots \int_{t_2}^\infty
(F_P-F_Q)(t_1) \, dt_1 \cdots dt_k,
\end{align*}
where in the last step we used the inductive hypothesis.  This completes the
proof. 

\bibliographystyle{plainnat}
\bibliography{veeru}

\end{document}